\documentclass[nohyperref]{article}

\usepackage{microtype}
\usepackage{graphicx}
\usepackage{subcaption}
\usepackage{booktabs} 

\usepackage{hyperref}

\usepackage[accepted]{icml2022}

\usepackage{amsmath}
\usepackage{amssymb}
\usepackage{mathtools}
\usepackage{amsthm}
\usepackage{dsfont}
\usepackage{multirow}
\usepackage{makecell}
\usepackage[capitalize,noabbrev]{cleveref}

\usepackage{enumitem}

\usepackage{aisecure-math}
\usepackage{xspace}
\newcommand{\method}{\texttt{TPC}\xspace}
\usepackage{hyperref}
\usepackage[textsize=tiny]{todonotes}

\usepackage{xspace}
\newcount\Comments  
\usepackage{color}
\definecolor{darkgreen}{rgb}{0,0.5,0}
\definecolor{darkblue}{rgb}{0,0,0.5}
\definecolor{purple}{rgb}{1,0,1}
\definecolor{gray}{rgb}{0.5,0.5,0.5}
\newcommand{\kibitz}[2]{\ifnum\Comments=0\textcolor{#1}{#2}\fi}

\ifnum\Comments=0\else\excludeversion{old}\fi
\ifnum\Comments=0\newenvironment{new}{\color{cyan}}{}\else\fi

\newcommand{\additive}{additive\xspace}
\newcommand{\resolve}{composable\xspace}
\newcommand{\difresolve}{indirectly composable\xspace}

\icmltitlerunning{Transformation-Specific Smoothing for Point Cloud Models}

\begin{document}

\twocolumn[
\icmltitle{\method: Transformation-Specific Smoothing for Point Cloud Models}



\icmlsetsymbol{equal}{*}

\begin{icmlauthorlist}
\icmlauthor{Wenda Chu}{thu}
\icmlauthor{Linyi Li}{uiuc}
\icmlauthor{Bo Li}{uiuc}
\end{icmlauthorlist}

\icmlaffiliation{thu}{Institute for Interdisciplinary Information Sciences, Tsinghua University, Beijing, P. R. China~(work done
during remote internship at UIUC)}
\icmlaffiliation{uiuc}{University of Illinois Urbana-Champaign~(UIUC), Illinois, USA}
\icmlcorrespondingauthor{Wenda Chu}{chuwd19@mails.tsinghua.edu.cn}
\icmlcorrespondingauthor{Linyi Li}{linyi2@illinois.edu}
\icmlcorrespondingauthor{Bo Li}{lbo@illinois.edu}

\icmlkeywords{Machine Learning, ICML}

\vskip 0.3in
]



\printAffiliationsAndNotice{}  

\begin{abstract}
Point cloud models with neural network architectures have achieved great success  and  have been widely used in safety-critical applications, such as Lidar-based recognition systems in autonomous vehicles.
However, such models are shown to be vulnerable to adversarial attacks that aim to apply stealthy semantic transformations such as rotation and tapering to mislead model predictions.
In this paper, we propose a transformation-specific smoothing framework \method, which provides \emph{tight} and \emph{scalable}  robustness guarantees for point cloud models against semantic transformation attacks. 
We first categorize common 3D transformations into three categories: \additive~(e.g., shearing), \resolve~(e.g., rotation), and \difresolve~(e.g., tapering), and we present generic robustness certification strategies for all categories respectively. We then specify unique certification protocols for a range of specific semantic transformations and their compositions.
Extensive experiments on several common 3D transformations show that \method significantly outperforms state of the art. For example, our framework boosts the certified accuracy against twisting transformation along the $z$-axis (within $\pm20^\circ$) from $20.3\%$ to $83.8\%$. Codes and models are available at \url{https://github.com/chuwd19/Point-Cloud-Smoothing}.
\end{abstract}

\section{Introduction}
\label{Introduction}

    Deep neural networks that take point clouds data as inputs~(point cloud models) are widely used in computer vision~\citep{pointnet,wang2019dynamic,zhou2018voxelnet} and autonomous driving~\citep{li20173d,chen2017multi,chen20203d}.
    For instance, modern autonomous driving systems are equipped with LiDAR sensors that generate point cloud inputs to feed into point cloud models~\citep{cao2019adversarial}.
    Despite their successes, point cloud models are shown to be vulnerable to adversarial attacks that mislead the model's prediction by adding stealthy perturbations to point coordinates or applying semantic transformations~(e.g., rotation, shearing, tapering)~\cite{cao2019adversarial,xiang2019generating,xiao2019meshadv,fang2021invisible}.
    Specifically, semantic transformation based attacks can be easily operated on point cloud models by simply manipulating sensor positions or orientations~\cite{cao2019adversarial,cao2021invisible}.
    These attacks may lead to severe consequences such as forcing an autonomous driving vehicle to steer toward the cliff~\cite{pei2017deepxplore}. 
    A wide range of empirical defenses against these attacks has been studied~\cite{zhu2017robust,aoki2019pointnetlk,sun2020towards,Sun2020On,sun2021adversarially}, while defenses with robustness guarantees are less explored~\cite{DeepG3D,Marc2021} and provides loose and less scalable certification.
    
    \begin{figure}[!t]
        \centering
        \includegraphics[width=\linewidth]{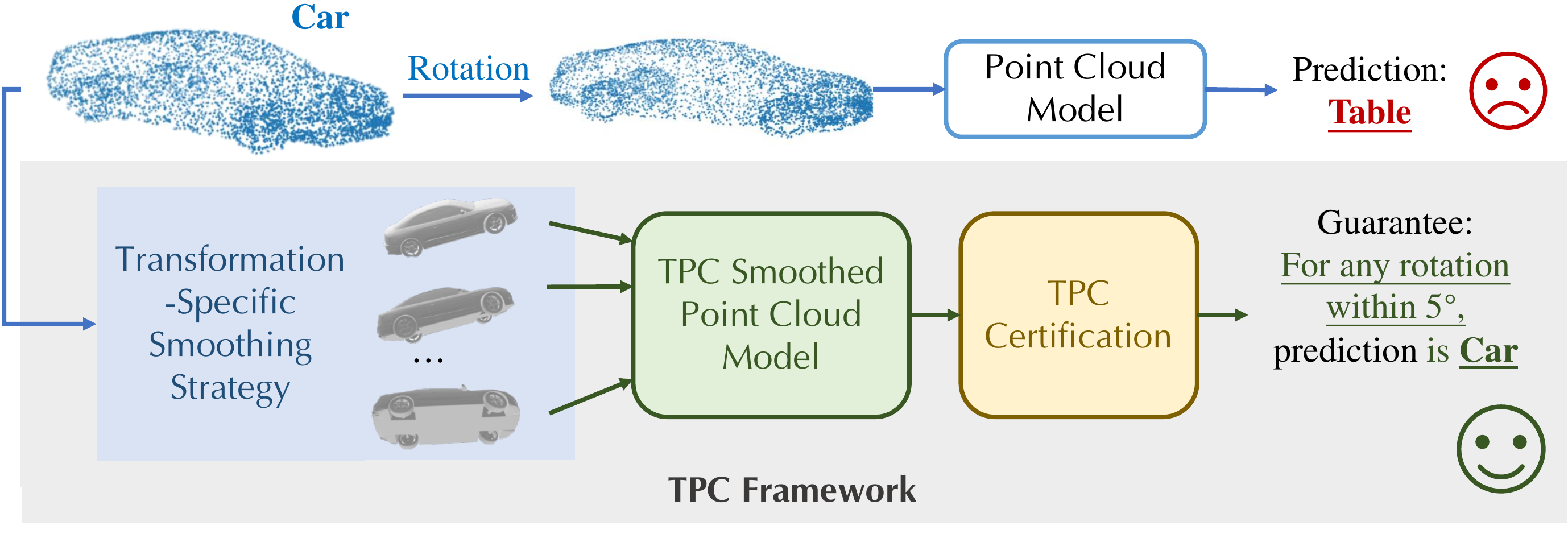}
        \vspace{-2em}
        \caption{Overview of \method framework. \method includes smoothing and certification strategies to provide certified robustness for point cloud models against semantic transformations.
        Besides the rotation as shown in the figure, \method provides strong robustness certification for a wide range of other semantic transformations.}
        \label{fig:1}
        \vspace{-1.5em}
    \end{figure}
    
    In this paper, we propose a transformation-specific smoothing framework \method that provides \textit{tight} and \textit{scalable} \textit{probabilistic} robustness guarantees for point cloud models against a wide range of semantic transformation attacks.
    We first categorize common semantic transformations into three categories: \additive~(e.g., shearing), \resolve~(e.g., rotation), and \difresolve~(e.g., tapering).
    For each category, our framework proposes novel \emph{smoothing} and \emph{robustness certification} strategies.
    With \method, for each common semantic transformation or composition, we prove the corresponding robustness conditions that yield efficient and tight robustness certification.
    
    For example, regarding general rotation based attacks, we first prove that it is a type of \textit{\resolve} transformations; we then propose a corresponding smoothing strategy and  certify the robustness of the smoothed model with a novel sampling-based algorithm, which is shown to have sound and tight input-dependent sampling error bound. \method achieves $69.2\%$ certified robust accuracy for any rotation within $10^\circ$. To the best of our knowledge, no prior work can provide robustness certification for rotations within such large angles. 
    
    In addition to our theoretical analysis for the certification against different types of semantic transformations, we conduct extensive experiments to evaluate \method.
    Compared with existing baselines, our \method achieves \emph{substantially} higher certified robust accuracy.
    For example, for any twisting along $z$-axis within $\pm 20^\circ$, we achieve the probabilistic certified accuracy of $83.8\%$ with a high confidence level $99.9\%$, while the existing baseline~\citep{DeepG3D} provides a deterministic certified accuracy of $20.3\%$.
    Furthermore, compared with prior works, we show that \method can: (1)~certify a more general class of semantic transformations; (2)~certify large-size point clouds; and (3)~certify under large perturbation magnitudes.
    We also show that \method can certify the robustness for multiple tasks on 3D point clouds, including classification and part segmentation.
    
    We illustrate our \method framework in \Cref{fig:1}, and we summarize the main technical \textbf{contributions} as follows.
    \begin{itemize}[leftmargin=*]
    \setlength{\parskip}{0pt}
    \vspace{-0.1in}
        \item We propose a general robustness certification framework \method for point cloud models. 
        We categorize common semantic transformations of point clouds into three categories: \additive, \resolve, and \difresolve, and provide general smoothing and certification strategies for each.
        
        \item We concretize our framework \method to provide transformation-specific smoothing and certification for various realistic common semantic transformations for point clouds, including rotation, shearing, twisting, and tapering as well as their compositions.
        
        \item We conduct extensive experiments and show that \method (1) achieves significantly higher certified robust accuracy than baselines, (2)  provides certification for large-size point clouds and large perturbation magnitudes, (3) provides efficient and effective certification for different tasks such as classification and part segmentation.
    \end{itemize}
    
    \subsection*{Related Work}
    
    \textbf{Certified Robustness of Deep Neural Networks.}
    To mitigate the threats of adversarial attacks on deep neural networks~\citep{szegedy2013intriguing,tramer2020adaptive,eykholt2018robust,qiu2020semanticadv,li2020qeba,zhang2021progressive,li2020nolinear,xiao2018characterizing}, efforts have been made toward certifying and improving the certified robustness of DNNs~\citep{cohen19c,li2020sok,li2019robustra}.
    Existing works mainly focus on image classification models against $\ell_p$ bounded perturbations.
    For such threat models, the robustness certification can be roughly divided into two types: deterministic and probabilistic,
    where deterministic methods are mainly based on feasible region relaxation~\citep{wong2018provable,weng2018towards,zhang2018efficient}, abstract interpretation~\citep{mirman2018differentiable,DeepPoly}, or Lipschitz bounds~\citep{tsuzuku2018lipschitz,zhang2021towards};
    and probabilistic methods provide certification that holds with high probability, and they are mainly based on randomized smoothing~\citep{cohen19c,yang2020randomized}.
    Along with the certification methods, there are several robust training methods that aim to train DNNs to be more certifiably robust~\citep{wong2018scaling,li2019robustra,salman2019provably}. 
    
    \textbf{Semantic Transformation Attacks and Certified Robustness on Point Cloud Models.}
    Our \method aims to generalize the model robustness certification to point cloud models against a more generic family of practical attacks -- semantic transformation attacks.
    The semantic transformation attacks have been shown feasible for both image classification models and point cloud models~\citep{hendrycks2018benchmarking,cao2019adversarial,xiang2019generating}, and certified robustness against such attacks is mainly studied for 2D image classification models~\citep{balunovic2019certifying,fischer2020certified,TSS}.
    For point cloud models,
    some work considers point \textit{addition} and \textit{removal} attacks~\citep{xiang2019generating} and provides robustness certification against such attacks~\citep{liu2021pointguard}. The randomized smoothing technique is applied to certify point cloud models on segmentation tasks by~\citep{Marc2021}. However, their certification only covers points edition with bounded $\ell_2$ norm and rotations along a fixed axis.
    For general semantic transformation attacks,
    to the best of our knowledge, the only work that can provide robustness certification against them is DeepG3D~\citep{DeepG3D}, which is based on linear bound relaxations.
    In this work, we derive novel randomized smoothing techniques on point clouds models to provide probabilistic robustness certification against semantic transformations.
    In \Cref{sec:exp}, we conduct extensive experiments to show that our framework is more general and provides significantly higher certified robust accuracy than DeepG3D under different settings.

\section{Semantic Transformation Attacks on Point Cloud Models}

We denote the space of point cloud inputs as $\mathcal X = \mathbb R^{N\times 3}$ where $N$ is the number of points the point cloud has. A point cloud with $N$ points is denoted by $x = \{p_i\}_{i=1}^N$ with $p_i \in \mathbb R^3$. Unless otherwise noted, we assume all point cloud inputs are normalized to be within a unit ball, i.e., $\|p_i\|_2 \leq 1$. 
We mainly consider classification tasks on the point clouds level.
Such classification task is defined with a set of labels $\mathcal Y = \{1,\dots, C\}$ and a classifier is defined by a deterministic function $h: \mathcal X \to \mathcal Y$.
More extensions are in \Cref{part-segmentation}.

\subsection{Semantic Transformations}

Semantic transformations on point cloud models are defined as functions $\phi: \mathcal X \times \mathcal Z \to \mathcal X$ where $\mathcal Z$ is the parameter space for transformations. The semantic transformations discussed in this paper may change the three-dimensional coordinate of each point (usually in a point-wise manner) but do not increase or decrease the number of points. 
In \Cref{section:3}, we will further categorize different semantic transformations based on their intrinsic properties.

\subsection{Threat Model and Certification Goal}
We consider semantic transformation attacks that an adversary can apply arbitrary semantic transformations to the point cloud data according to a parameter $z\in \mathcal Z$. The adversary then performs evasion attacks to a classifier $h$ with the transformed point cloud $\phi(x,z)$. The attack is successful if $h$ predicts different labels on $x$ and $\phi(x,z)$\footnote{Without loss of generality, we consider untargeted attacks here, and the targeted attack can be derived similarly.}.

The main goal of this paper is to certify the robustness of point cloud classifiers against all semantic attacks within a certain transformation parameter space. Formally, our \textbf{certification goal} is to find a subset $\mathcal Z_{\mathrm{robust}}\subseteq Z$ for a classifier $h:\mathcal X \to \mathcal Y$, such that
\vspace{-0.05in}
\begin{equation}
    h(x) = h(\phi(x,z)), \forall z \in \mathcal Z_{\mathrm{robust}}.
    \label{eq:robustness}
\end{equation}
\section{Transformation Specific Smoothing for Point Cloud Models}\label{section:3}

In this section, we first introduce the proposed randomized smoothing techniques for general semantic transformations. Next, we categorize the semantic transformations into three types: \resolve, additive, and \difresolve transformations. We then derive the smoothing-based certification strategies for each type.

\begin{figure}[tb]
    \centering
    \begin{subfigure}[b]{0.45\linewidth}
    \centering
    \includegraphics[width=\textwidth]{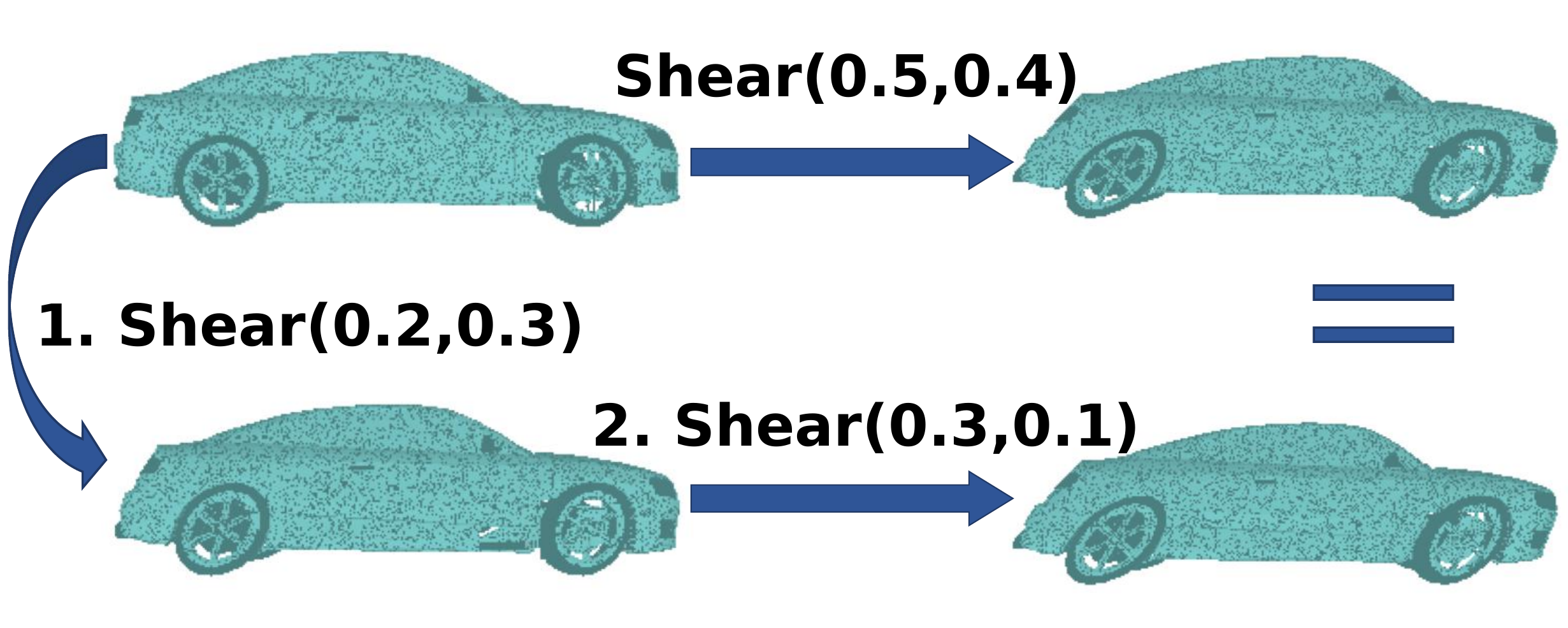}
    \caption{Additive}
    \end{subfigure}
    \begin{subfigure}[b]{0.45\linewidth}
    \centering
    \includegraphics[width=\textwidth]{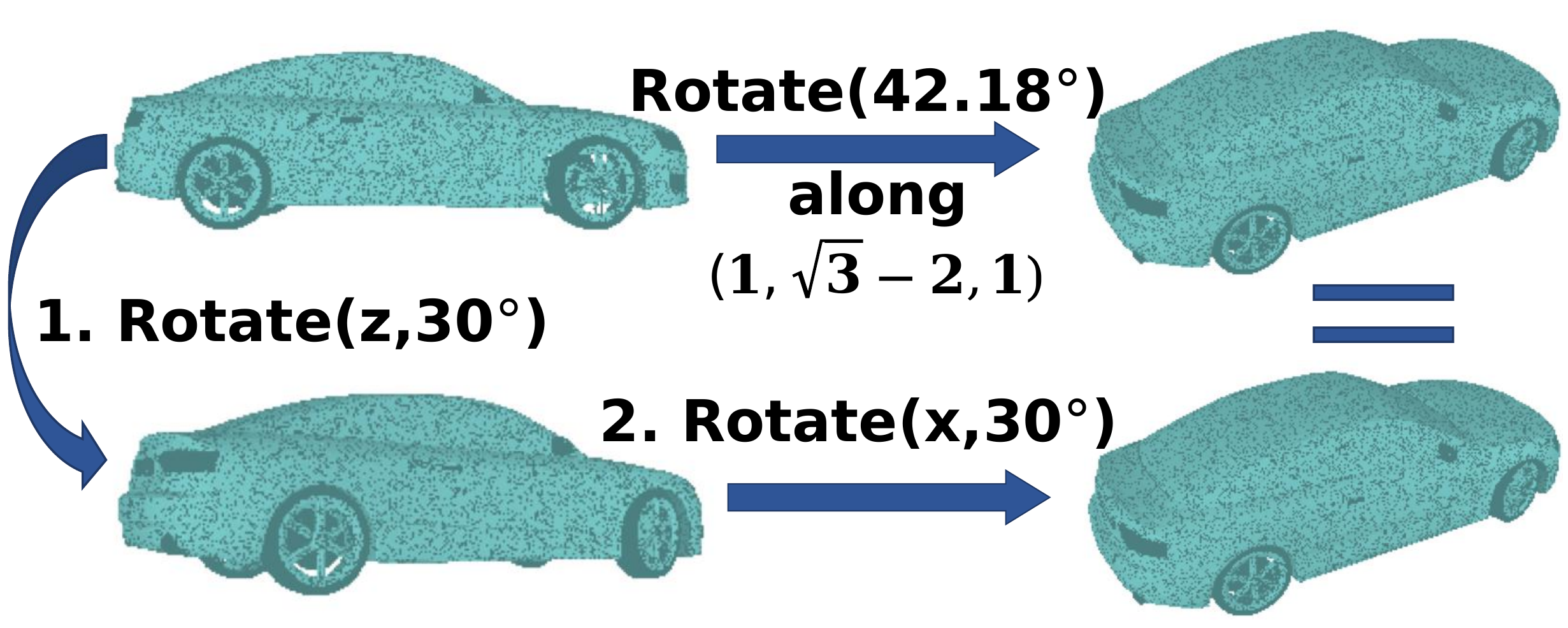}
    \caption{Composable
    }
    \end{subfigure}
    \begin{subfigure}[b]{0.45\linewidth}
    \centering
    \includegraphics[width=\textwidth]{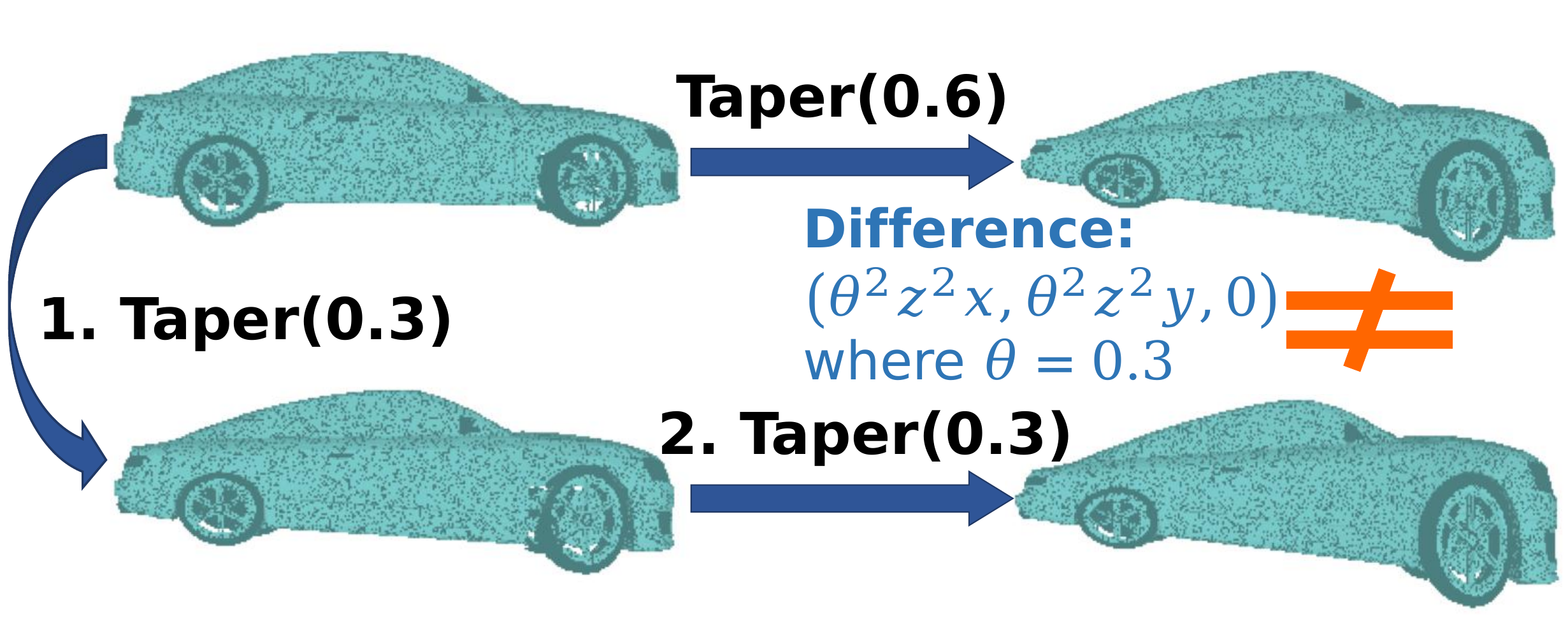}
    \caption{Indirectly composable}
    \end{subfigure}
    \vspace{-0.1in}
    \caption{Illustration of different types of transformations. (a) additive transformations (e.g., shearing), (b) \resolve transformations (e.g., rotation) and (c) \difresolve transformations (e.g., tapering).}
    \label{fig:2}
    \vspace{-0.1in}
    \label{fig-transformation}
\end{figure}

\begin{figure}[!t]
    \begin{center}
    \centerline{\includegraphics[width = .9 \linewidth]{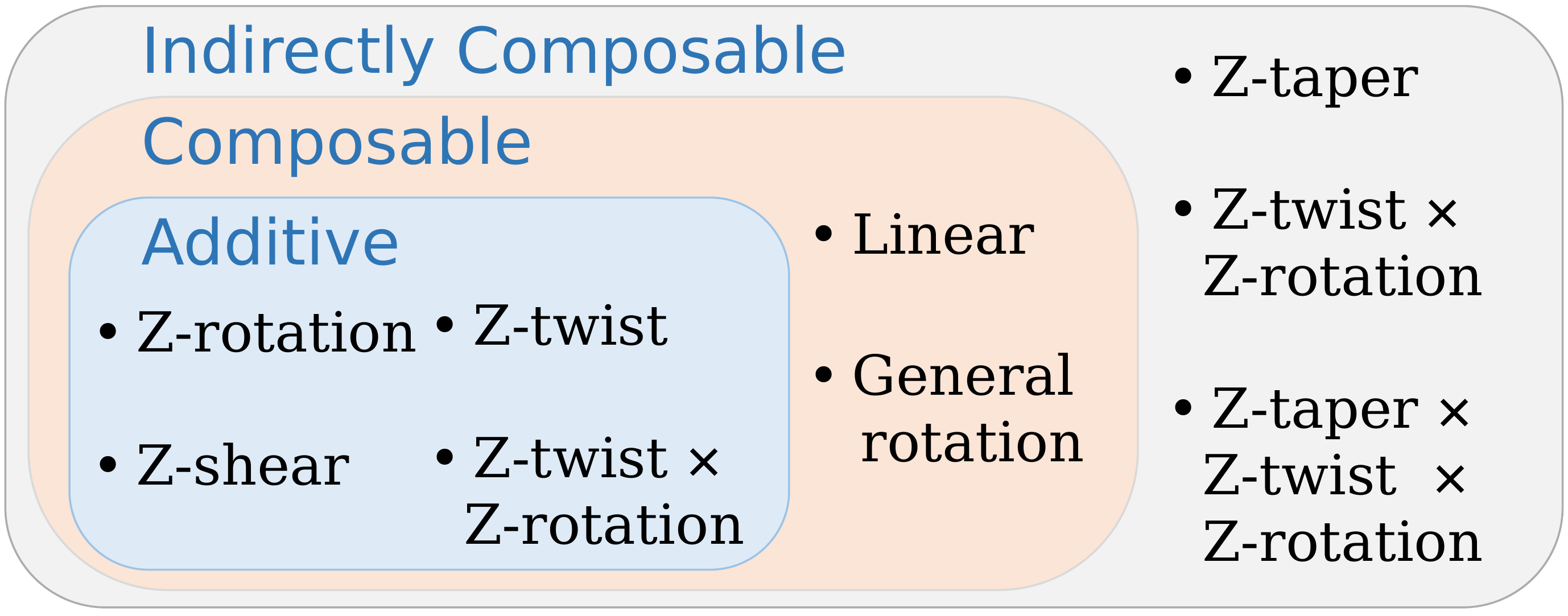}}
    \vspace{-0.05in}
    \caption{Taxonomy of common 3D semantic transformations for point clouds.}
    \label{fig-category}
    \end{center}
    \vspace{-0.4in}
\end{figure}
\subsection{Transformation Specific Smoothed Classifier}

We apply transformation-specific smoothing to an arbitrary base classifier $h:\mathcal X \to \mathcal Y$ to construct a smoothed classifier. Specifically, the smoothed classifier $g$ predicts the class with the highest conditional probability when the input $x$ is perturbed by some random transformations.
\begin{definition}[Transformation Specific Smoothed Classifier]
Let $\phi:\mathcal X \times \mathcal Z \to \mathcal X$ be a semantic transformation. Let $\epsilon$ be a random variable in the parameter space $\mathcal Z$. Suppose we have a base classifier that learns a conditional probability distribution, $h(x) = \arg\max_{y\in \mathcal Y} p(y|x)$. Applying transformation specific smoothing to the base classifier $h$ yields a smoothed classifier $g:\mathcal X \to\mathcal Z$, which predicts
\begin{equation}
\small
g(x;\epsilon) = \arg\max_{y\in \mathcal Y} q(y|x,\epsilon) = \arg\max_{y\in \mathcal Y} \mathbb E_{\epsilon}(p(y|\phi(x,\epsilon))).
\end{equation}
\end{definition}

\vspace{-0.05in}
We recall the theorem proved  \cite{TSS} in \cref{appendix:generic}, which provides a generic certification bound for the transformation-specific smoothed classifier based on the Neyman-Pearson lemma~\cite{neyman1933ix}. 

Next, we will categorize the semantic transformations into different categories based on their intrinsic properties as shown in Figure~\ref{fig-category}, and we will then discuss the certification principles for each specific category. 

\subsection{Composable Transformations}
\label{resolvable}

A set of semantic transformations is called \resolve if it is closed under composition.
\begin{definition}
\label{def:composable}
A set of semantic transformations defined by $\phi:\mathcal X\times \mathcal Z\to \mathcal X$ is called \textbf{\resolve} if for any $\alpha \in \mathcal Z$ there exists an injective and continuously differentiable function $\gamma_\alpha: \mathcal Z\to \mathcal Z$ with non-vanishing Jacobian, such that
\vspace{-0.05in}
\begin{equation}
\label{eq:composable}
    \phi(\phi(x,\alpha),\beta) = \phi(x,\gamma_\alpha(\beta)),\ \forall x\in \mathcal X, \beta \in \mathcal Z.
\end{equation}
\end{definition}

\vspace{-0.1in}
Common semantic transformations for point cloud data that are \resolve include rotation, shearing along a fixed axis, and twisting along a fixed axis. For example, according to Euler's rotation theorem, we can always find another rotation $\gamma_\alpha(\beta)\in \mathcal Z$ for any two rotations $\alpha, \beta\in\mathcal Z$.  Therefore, rotations belong to the \resolve transformations as shown in Figure~\ref{fig-transformation} (b).

In general, \resolve transformations can be certified against by \cref{thm:generic} stated in \cref{appendix:generic}.  For a classifier $g(x;\epsilon_0)$ smoothed by the \resolve transformation, we can simply replace the random variable $\epsilon_1$ by $\gamma_{\alpha}(\epsilon)$ in \cref{thm:generic} to derive a robustness certification condition. However, some \resolve transformations with complicated $\gamma_\alpha(\beta)$ function result in intractable distribution for $\epsilon_1$, causing difficulties for the certification. Therefore, we focus on a subset of \resolve transformations, called \emph{additive transformations}, for which it is straight-forward to certify by applying \cref{thm:generic}.



\subsection{Additive Transformations}
We are particularly interested in a subset of \resolve transformations that the function $\gamma_\alpha:\mathcal Z \to \mathcal Z$ defined in \cref{def:composable} satisfies $\gamma_\alpha(\beta) = \alpha + \beta$ as shown in Figure~\ref{fig-transformation}~(a) where the one step rotation above is equivalent to the two step transformations below. 
\begin{definition}
A set of semantic transformations $\phi:\mathcal X\times \mathcal Z \to \mathcal X$ is called \textbf{additive} if
\vspace{-0.05in}
\begin{equation}
    \phi(\phi(x,\alpha),\beta) = \phi(x,\alpha + \beta),\ \forall\ x\in \mathcal X, \alpha, \beta\in \mathcal Z.
\end{equation}
\end{definition}
\vspace{-0.1in}
An additive transformation must be \resolve, but the reverse direction does not hold. For instance, the set of general rotations from the SO(3) group is \resolve, but not additive. Rotating 10$^\circ$ along the $x$ axis first and then 10$^\circ$ along the $y$ axis does not equal rotating 20$^\circ$ along the $xy$ axis. Thus, general rotations cannot be categorized as an additive transformation. However, rotating along any fixed axis is additive. Based on this observation, we discuss z-rotation (i.e., rotation along $z$ axis) and general rotations separately in \cref{section:4}.

All additive transformations can be certified  following  the same protocol derived from \cref{thm:generic}. We omit \cref{additive+gaussian} for certified robustness against additive transformation in \cref{appendix:generic}.

\subsection{Indirectly Composable Transformations}
\label{differentially_resolvable}

As shown in \cref{resolvable}, \resolve transformations can be certified following \cref{thm:generic}. However, some semantic transformations of point clouds do not have such closure property under composition and thus do not fall in this category as shown in Figure~\ref{fig-transformation} (c). For example, the tapering transformation which we will discuss in \cref{section:4} is not \resolve and cannot be certified  directly using \cref{thm:generic}. This kind of transformation is therefore categorized as a more general class called \difresolve transformations.
\begin{definition}
A set of transformations $\phi: \mathcal X \times \mathcal Z_{\phi} \to \mathcal X$ is \textbf{\difresolve} if there is a set of \resolve transformations $\psi:\mathcal X\times \mathcal Z_\psi \to \mathcal X$, such that for any $x\in \mathcal X$, there exists a function $\delta_x: \mathcal Z_\phi \times \mathcal Z_\phi \to \mathcal Z_\psi$ with
\vspace{-0.05in}
\begin{equation}
    \psi(\phi(x,\alpha), \delta_x(\alpha,\beta))
    =
    \phi(x,\beta), \ \forall \alpha,\beta\in \mathcal Z_\phi.
\end{equation}
\end{definition}

\vspace{-0.1in}
This definition involves more kinds of transformations, since we can choose the transformation $\psi$ as $\psi(x,\delta) = x + \delta$ and let $\delta_x(\alpha,\beta) = 
\phi(x,\beta) -\phi(x,\alpha)
$. This specific assignment of $\psi$ leads to a useful theorem \cite{TSS} in \cref{appendix:TSS-DR}, which we use to certify against some more complicated transformations, such as tapering in \cref{section:4}. The theorem states that the overall robustness can be guaranteed if we draw multiple samples within the parameter space and certify the neighboring distribution of each sampled parameter separately.


\section{Certifying Point Cloud Models against Specific Semantic Transformations}
\label{section:4}

In this section, we certify the point cloud models against several specific semantic transformations that are commonly seen for point cloud data, including rotation, shearing, twisting, and tapering. We do not analyze scaling and translation, since the point cloud models are usually inherently invariant to them due to the standard pre-processing pipeline \cite{pointnet}. For each transformation, we specify a corresponding certification protocol based on the categorization they belong to introduced in \cref{section:3}.

\subsection{Rotation, Shearing, and Twisting along a Fixed Axis}

Rotation, shearing, and twisting are all common 3D transformations that are performed pointwise on point clouds. Without loss of generality, we consider performing these transformations along the $z$-axis.

Specifically, we define \textbf{z-shear} transformation as $\phi_{Sz}: \mathcal X\times \mathcal Z \to \mathcal X$ where $\mathcal X = \mathbb R^{N\times 3}$ is the space of the point clouds with $N$ points and $\mathcal Z = \mathbb R^2$ is the parameter space. For any $z = (\theta_1,\theta_2)$, z-shear acting on a point cloud $x\in \mathcal X$ with $x = \{p_i\}_{i=1}^N$ yields ($p_i = (x_i,y_i,z_i)^T$)
\vspace{-0.03in}
\begin{equation}
    \phi_{Sz}(p_i, z) = 
    (x_i + \theta_1 z_i, y_i + \theta_2 z_i, z_i).
\end{equation}

\vspace{-0.1in}
\textbf{Z-twist} transformation $\phi_{Tz}:\mathcal X \times \mathcal Z \to \mathcal X$ is defined similarly but with parameter space $\mathcal Z = \mathbb R$. For any $\theta\in Z$ and $p_i = (x_i,y_i,z_i)^T$,
\vspace{-0.03in}
\begin{equation}
\small
    \phi_{Tz}(p_i, \theta) = \begin{pmatrix}
    x_i\cos(\theta z_i) - y_i \sin(\theta z_i)\\
    x_i \sin(\theta z_i) + y_i\cos(\theta z_i)\\
    z_i
    \end{pmatrix}.
\end{equation}
Note that z-rotation, z-shear and z-twist are all \emph{additive} transformations. Hence, we present the following corollary based on \cref{additive+gaussian}, which certifies the robustness of point clouds models with bounded $\ell_2$ norm for the transformation parameters.
\begin{corollary}
Suppose a classifier $g:\mathcal X\to \mathcal Y$ is smoothed by a transformation $\phi:\mathcal X \times \mathcal Z \to \mathcal X$ with $\epsilon \sim \mathcal N(0,\sigma^2\mathds 1_d)$. Assume its class probability satisfies \cref{class_probability}. If the transformation is z-rotation, z-shear, or z-twist  ($\phi = \phi_{Sz}, \phi_{Tz}$ or $\phi_{Rot-z}$), then it is guaranteed that $g(\phi(x,\alpha);\epsilon) = g(x;\epsilon)$, if the following condition holds:
\begin{equation}
    \|\alpha\|_2 \leq \frac{\sigma}{2}\Big(\Phi^{-1}(p_A) - \Phi^{-1}(p_B)\Big), \alpha \in \mathcal Z.
\end{equation}
\end{corollary}

\subsection{Tapering along a Fixed Axis}

Tapering a point keeps the coordinate of a specific axis $k$, but scales the coordinates of other axes proportional to $k$'s coordinate. For clarity, we define \textbf{z-taper} transformation $\phi_{TP}:\mathcal X\times \mathcal Z \to \mathcal X$ as tapering along the $z$-axis, with its parameter space defined by $\mathcal Z = \mathbb R$. For any point cloud $x = \{p_i\}_{i=1}^N\in \mathcal X$ ($p_i = (x_i,y_i,z_i)$) and for any $\theta \in \mathcal Z$,
\vspace{-0.05in}
\begin{equation}
    \phi_{TP}(p_i,\theta) =
    \big(x_i (1+\theta z_i),
    y_i (1+\theta z_i),
    z_i\big).
    \label{eq:z-taper}
\end{equation}

\vspace{-0.15in}
However, z-taper is not a \resolve transformation, since the composition of two z-taper transformations contains terms with $z_i^2$ component. Therefore, we propose a specific certification protocol for z-taper based on \cref{thm:TSS-DR}. To achieve this goal, we specify a sampling strategy in the parameter space $\mathcal Z$ and bound the interpolation error (\cref{eq:interpolation}) of the sampled z-taper transformations.
\begin{theorem}
Let $\phi_{TP}: \mathcal X \times \mathbb R \to \mathcal X$ be a z-taper transformation. Let $g: \mathcal X \to \mathcal Y$ be a $\epsilon$-smoothed classifier with random noises $\epsilon \sim \mathcal N(0,\sigma^2 \mathds 1_{3\times N})$, which predicts $g(x;\epsilon) = \arg\max_y q(y|x;\epsilon) = \arg\max\mathbb E_\epsilon p(y| x+\epsilon)$.
\label{thm:taper}
Let $\{\theta_j\}_{j=0}^M$ be a set of transformation parameters and $\theta_j = (\frac{2j}{M} - 1)R$. Suppose for any $i$,
\vspace{-0.05in}
\begin{equation}
\small
    q(y_A|\phi_{TP}(x,\theta_j);\epsilon) \geq p_A^{(j)} > p_B^{(j)} \geq \max_{y\neq y_A} q(y|\phi_{TP}(x,\theta_j); \epsilon)
\end{equation}
Then it is guaranteed that $\forall \theta\in [-R,R]$: $y_A = \arg\max_y q(y|\phi_{TP}(x,\theta); \epsilon)$ if for all $j = 1,\dots,M$, 
\vspace{-0.05in}
\begin{equation}
\small
    \frac{\sigma}{2} \left(\Phi^{-1} \left(p_A^{(j)}\right) - \Phi^{-1}\left(p_B^{(j)}\right)\right) \geq  \frac{R\sqrt{N}}{2M}.
\end{equation}
\end{theorem}
\vspace{-0.8em}
Detailed proof for \cref{thm:taper} can be found in \cref{appendix:taper}.

\subsection{General Rotation}
\label{general-rotation-section}
Rotation is one of the most common transformations for point cloud data. Therefore, we hope the classifier is robust not only against rotation attacks along a fixed axis, but also those along arbitrary axes. In this section, we first define general rotation and show its universality for rotations as well as their composition; and then provide a concrete certification protocol for smoothing and certifying the robustness against this type of transformation.

We define \textbf{general rotation} transformations as $\phi_R: \mathcal X \times \mathcal Z \to \mathcal X$ where $\mathcal Z = S^2 \times \mathbb R^+$ is the parameter space of rotations. For a rotation $z\in \mathcal Z$, its rotation axis is defined by a unit vector $k\in S^2$ and its rotation angle is $\theta \in \mathbb R^+$. For any 3D point $p_i\in \mathbb R^3$,
\vspace{-0.05in}
\begin{equation}
    \phi_R(p_i, z) = Rot(k,\theta) p_i,\ z = (k,\theta).
    \label{eq:general_rotation}
\end{equation}

\vspace{-0.1in}
where $Rot(k,\theta)$ is the rotation matrix that rotates by $\theta$ along axis $k$. General rotations are \emph{\resolve} transformations since the composition of any two 3D rotations can be expressed by another 3D rotation. 

However, certifying against the general rotation is more challenging, since the general rotation is not additive and the expression of their composition is extremely complicated. In particular, if we smooth a base classifier with a random variable $\epsilon_0$, a semantic attack with parameter $\alpha\in \mathcal Z$ results in $\phi_R(\phi_R(x,\alpha), \epsilon_0) = \gamma_\alpha(\epsilon_0)$, which is a bizarre distribution in the parameter space. Therefore, we cannot directly apply \cref{thm:generic} to certify general rotation.


On the other hand, as \cref{thm:TSS-DR} shows, if we uniformly sample many parameters in a subspace of $\mathcal Z = S^2 \times \mathbb R^+$ and certify robustness in the neighborhood of each sample, we are able to certify a large and continuous subspace $\mathcal Z_{\mathrm{robust}} \subseteq \mathcal Z$. As a result, we propose a sampling-based certification strategy, together with a tight bound for the interpolation error of general rotation transformations, which we summarize in the following theorem.

\begin{theorem}
\label{thm:general_rotation}
Let $\phi_{R}: \mathcal X \times \mathcal Z \to \mathcal X$ be a general rotation transformation. Let $g: \mathcal X \to \mathcal Y$ be a classifier smoothed by random noises $\epsilon \sim \mathcal N (0,\sigma^2 \mathds 1_{3\times N})$, which predicts $g(x; \epsilon) = \arg\max_y q(y|x;\epsilon) = \arg\max_y \mathbb E(p(y|x+\epsilon))$.
Let $\{z_j\}_{j=1}^M$ be a set of transformation parameters with $z_j = (k_j, \theta_j), k_j \in S^2, \theta_j \in \mathbb R^+$ such that
\vspace{-0.05in}
\begin{equation}
\label{eq:rotation_condition}
\small
    \forall k\in S^2, \theta \in [0,R], \exists k_j, 
    \theta_j \,\mathrm{s.t.}\, \langle k, k_j\rangle \leq \epsilon, |\theta - \theta_j| \leq \delta
\end{equation}

\vspace{-0.13in}
Suppose for any $j$, the smoothed classifier $g$ has class probabilities that satisfy
\vspace{-0.05in}
\begin{equation}
\small
    q(y_A|\phi_R(x,z_j); \epsilon) \geq p_A^{(j)} > p_B^{(j)}\geq \max_{y\neq y_A} q(y|\phi_R(x,z_j); \epsilon).
\end{equation}

\vspace{-0.13in}
Then it is guaranteed that for any $z$ with rotation angle $\theta < R$: $y_A = \arg\max_y q(y|\phi_R(x,z);\epsilon)$ if $\ \forall j$,
\vspace{-0.05in}
\begin{equation}
\small
    \frac{\sigma}{2} \left(\Phi^{-1} \left(p_A^{(j)}\right) - \Phi^{-1}\left(p_B^{(j)}\right)\right) 
    \geq \pi \sqrt{\frac{\delta^2}{4} + \frac{\epsilon^2 R^2}{8}} \|x\|_2.
\end{equation}
\end{theorem}

\vspace{-0.13in}
We present a proof sketch here and leave the details in \cref{appendix:general_rot}. Notice that the interpolation error between two transformations on a point cloud $x=\{p_i\}_{i=1}^N$ can be calculated by $\|\phi(x,z_j) - \phi(x,z)\|_2 = \|\phi(x,z^\prime) - x\|_2 \leq \theta^{\prime} (\sum_i^N \|p_i\|_2^2)^{1/2}$, where $z^\prime = (k^\prime, \theta^\prime)$ is the composition of the rotation with parameter $z$ and the reverse rotation $z_j^{-1}$. Combined with the generic theorem for \difresolve transformations (\cref{thm:TSS-DR}), bounding $\theta^\prime$ using \cref{eq:rotation_condition} yields \cref{thm:general_rotation}.

\subsection{Linear Transformations}
\label{section:linear}
Here, we consider a broader class of semantic transformations that contains all linear transformations applied to a 3D point. Formally, a \textbf{linear} transformation $\phi_L: \mathcal X\times \mathcal Z \to  \mathcal X$ has a parameter space of $\mathcal Z = \mathbb R^{3\times 3}$. For any point cloud $x = \{p_i\}_{i=1}^N\in \mathcal X$ and for any $\mathbf A\in \mathcal Z$,
\vspace{-0.05in}
\begin{equation}
\label{def:linear}
    \phi_{L}(p_i, \mathbf A) = (\mathbf I+\mathbf A)p_i.
\end{equation}

\vspace{-0.15in}
\cref{def:linear} describes any linear transformation with a bounded perturbation $\mathbf A$ from the identity transformation $\mathbf I$. A natural threat model is considered by \cite{Affine2020} that the perturbation matrix $\mathbf A$ has a bounded Frobenius norm $\|\mathbf A\|_F \leq \epsilon$, for which we present a certification protocol in this paper.  Linear transformations are \resolve because their compositions are also linear, but the fact that they are not additive prohibits a direct usage of \cref{additive+gaussian}. Nevertheless, these transformations can still be certified with a more complicated protocol, if Gaussian smoothing is applied.

\begin{theorem}
\label{thm:linear}
Suppose a classifier $g$ is smoothed by random linear transformations $\phi_L:\mathcal X \times \mathcal Z \to \mathcal X$ where $\mathcal Z = \mathbb R^{3\times 3}$, with a Gaussian random variable $\epsilon \sim \mathcal N(0,\sigma^2\mathbf I_9)$. If the class probability satisfies \cref{class_probability}, then it is guaranteed that $g(\phi((x,\alpha);\epsilon) = g(x;\epsilon)$ for all $\|\alpha\|_F \leq R$, where
\vspace{-0.1in}
\begin{equation}
    R = \frac{\sigma \Big(\Phi^{-1}(\tilde p_A) - \Phi^{-1}(1 - \tilde p_A)\Big)}{2 + \sigma \Big(\Phi^{-1}(\tilde p_A) - \Phi^{-1}(1 - \tilde p_A)\Big)}.
\end{equation}

\vspace{-0.13in}
$\tilde p_A$ is a function of $p_A$ as explained in \cref{lemma:pAtilde}.

\end{theorem}

\subsection{Compositions of Different Transformations}
In addition to certifying against a single transformation, we also provide certification protocols for composite transformations, including $\text{z-twist} \circ \text{z-rotation}$, $\text{z-taper} \circ \text{z-rotation}$ and $\text{z-twist}\circ \text{z-taper}\circ \text{z-rotation}$.

Notice that $\text{z-twist} \circ \text{z-rotation}$ is an additive function:
\vspace{-0.08in}
\begin{align}
\small
    &\phi_{Tz}(\phi_{Rot-z}(\phi_{Tz}(\phi_{Rot-z}(x,\theta_1),\alpha_1),\theta_2),\alpha_2)\nonumber \\
    &= \phi_{Tz}(\phi_{Rot-z}(x,\theta_1 + \theta_2),\alpha_1 + \alpha_2).
\end{align}

\vspace{-0.15in}
Therefore, we directly apply \cref{additive+gaussian} in \cref{appendix:generic} to certify $\text{z-twist} \circ \text{z-rotation}$ transformation. The concrete corollary is stated as below.
\begin{corollary}
Suppose a classifier $g:\mathcal X\to \mathcal Y$ is smoothed by random transformations $\text{z-twist} \circ \text{z-rotation}$ $\phi:\mathcal X \times \mathcal Z \to \mathcal X$ where the parameter space $\mathcal Z = \mathcal Z_{Twist} \times \mathcal Z_{Rot-z} = \mathbb R^2$.
The random variable for smoothing is $\epsilon \sim \mathcal N(0,\mathrm{diag} (\sigma_1^2, \sigma_2^2))$. If the class probability of $g$ satisfies \cref{class_probability}, then it is guaranteed that $g(\phi(x,\alpha);\epsilon) = g(x;\epsilon)$ for all $(\alpha_1,\alpha_2)\in \mathcal Z$, if the following condition holds:
\vspace{-0.05in}
\begin{equation}
\small
    \sqrt{\left(\frac{\alpha_1}{\sigma_1}\right)^2 + \left(\frac{\alpha_2}{\sigma_2}\right)^2} \leq \frac{\sigma}{2}\Big(\Phi^{-1}(p_A) - \Phi^{-1}(p_B)\Big).
\end{equation}
\end{corollary}

\vspace{-0.13in}
Another composite transformation $\text{z-taper}\circ \text{z-rotation}$ first rotates the point cloud along z-axis, and then taper along z-axis. As z-taper is not \resolve with itself, this composite transformation is also not \resolve. Similar to z-taper, we certify the composite transformation $\text{z-taper}\circ \text{z-rotation}$ by upper-bounding the interpolation error in \cref{eq:interpolation}.

\begin{theorem}
\label{thm:taper+rotation}
We denote z-taper $\circ$ z-rotation by $\phi:\mathcal X\times \mathcal Z \to \mathcal X, \phi = \phi_{TP} \circ \phi_{Rot-z}$ with a parameter space of $\mathcal Z = \mathcal Z_{TP} \times \mathcal Z_{Rot-z} = \mathbb R^2$. Let $g:\mathcal X \to \mathcal Y$ be a classifier smoothed by random noises $\epsilon\sim \mathcal N(0,\sigma^2\mathds 1_{3\times N})$.

For a subspace in the parameter space, $S = [-\varphi, \varphi] \times [-\theta, \theta] \subseteq \mathcal Z$, we uniformly sample $\varphi\theta M^2$ parameters $\{z_{jk}\}$ in $S$. That is, $z_{jk} = (\varphi_j, \theta_k)$ where $\varphi_j = \frac{2j}{M} - \varphi$ and $\theta_k = \frac{2k}{M} - \theta$. Suppose for any $j,k$ the smoothed classifier $g$ has class probability that satisfy
\vspace{-0.05in}
\begin{equation}
\small
    q(y_A|\phi(x,z_{jk});\epsilon) \geq p_A^{(jk)} > p_B^{(jk)}\geq \max_{y\neq y_A} q(y|\phi(x,z_{jk});\epsilon),
\end{equation}

\vspace{-0.13in}
then it is guaranteed that $y_A = \arg\max_y q(y|\phi(x,z);\epsilon)$ if $\ \forall j,k$ and $\forall z\in S$,
\vspace{-0.05in}
\begin{equation}
\small
    \frac{\sigma}{2} \left(\Phi^{-1} \left(p_A^{(jk)}\right) - \Phi^{-1}\left(p_B^{(jk)}\right)\right)
    \geq \frac{\sqrt{N(4\varphi^2+8\varphi+5)}}{2M}.
\end{equation}
\end{theorem}

\vspace{-0.13in}
We also consider the composition of three transformations: z-twist $\circ$ z-taper $\circ$ z-rotation. 

\begin{theorem}
\label{thm:twist+taper+rotation}
We define the composite transformation z-twist $\circ$ z-taper $\circ$ z-rotation by $\phi:\mathcal X \times \mathcal Z \to \mathcal X$, with input space $\mathcal X = \mathbb R^{3\times N}$ and parameter space $\mathcal Z = \mathcal Z_{Twist} \times \mathcal Z_{Taper} \times \mathcal Z_{Rot-z} = \mathbb R^3$. Let $g:\mathcal X \to \mathcal Y$ be a classifier smoothed by random noises $\epsilon\sim \mathcal N(0,\sigma^2 \mathds 1_{3\times N})$, which predicts $g(x;\epsilon) = \arg\max_y q(y|x;\epsilon) = \arg\max_y \mathbb E(p(y|x+\epsilon))$.

Let $\{z_{jkl} \in \mathcal Z: z = (\varphi_j, \alpha_k, \theta_l)\}$ be a set of parameters with $\varphi_j = \frac{2j}{M} - \varphi$, $\alpha_k = \frac{2k}{M} - \alpha$ and $\theta_l = \frac{2l}{M} - \theta$. Therefore $(\varphi_j,\alpha_k, \theta_l)$ distribute uniformly in the subspace $\mathcal Z_{robust} = [-\varphi, \varphi]\times [-\alpha,\alpha]\times [-\theta, \theta] \subseteq \mathcal Z$. Suppose for any $j,k,l$, the smoothed classifier $g$ has class probability that satisfy
\vspace{-0.08in}
\begin{align}
\small
    q(y_A|\phi(x,z_{jkl}); \epsilon)&\geq p_A^{(jkl)}> p_B^{(jkl)}\nonumber \\
    &\geq \max_{y\neq y_A} q(y|\phi(x,z_{jkl});\epsilon),
\end{align}

\vspace{-1em}
then it is guaranteed that for any $z\in \mathcal Z_{robust}: y_A = \arg\max_y q(y|\phi(x,z);\epsilon)$, if for any $i,j,k$,
\vspace{-0.08in}
\begin{equation}
\small
    \frac{\sigma}{2}\Big(\Phi^{-1}\Big(p_A^{(jkl)}\Big) - \Phi^{-1}\Big(p_B^{(jkl)}\Big) \Big) \geq \frac{\sqrt{N(1 + \frac{27}{4}(1+\alpha)^2)}}{2M}
\end{equation}
\end{theorem}

\vspace{-0.13in}
Both \cref{thm:taper+rotation} and \cref{thm:twist+taper+rotation} are based on our proposed approach of sampling parameters in the parameter space and certifying the neighboring distributions of the samples separately by bounding the interpolation error~(\cref{eq:interpolation}). These two theorems are rigorously proved in \cref{appendix:taper+rotation} and \cref{appendix:twist+taper+rotation}.

\section{Experiments}
    \label{sec:exp}
We conduct extensive experiments on different 3D semantic transformations and models to evaluate the certified robustness derived from our \method framework. We show that \method significantly outperforms the state-of-the-art in terms of the certified robustness against a range of semantic transformations, and the results also lead to some interesting findings.

\subsection{Experimental setup}
\textbf{Dataset.} We perform experiments on the ModelNet40 dataset \cite{shapenet}, which includes different 3D objects of 40 categories. We follow the standard pre-processing pipeline that places the point clouds in the center and scales them into a unit sphere.

We also conduct experiments for part segmentation tasks, for which the ShapeNet dataset \cite{shapenet-dataset} is used for evaluation. It contains 16681 meshes from 16 categories and also 50 predefined part labels. The experiment results are presented in \cref{part-segmentation}.

\textbf{Models.} We run our experiments for point cloud classification on PointNet models \cite{pointnet} with different point cloud sizes. We apply data augmentation training for each transformation combined with consistency regularization to train base classifiers. We then employ our \method framework to smooth these models and derive robustness certification bounds against various transformations. \method does not depend on specific model selection and can be directly applied to certify other point cloud model architectures. We present certification results for other architectures (e.g., CurveNet \cite{curvenet}) in \cref{appendix:curvenet}.

\textbf{Evaluation Metrics.} To evaluate the robustness of point clouds classification, we pick a fixed random subset of the ModelNet40 test dataset. We report the \textbf{certified accuracy} defined by the fraction of point clouds that are classified both \emph{correctly} and \emph{consistently} within certain transformation space. The baseline we compare with \cite{DeepG3D} only presents \textbf{certified ratio}, which is the fraction of test samples classified \emph{consistently}. We believe that the \emph{certified accuracy} is a more rigorous metric for evaluation based on existing standard certification protocols in the image domain~\cite{cohen19c}. We thus calculate the certified accuracy for baselines based on the results reported in the paper~\cite{DeepG3D} for comparison. Besides, we also report the certified ratio comparison in \cref{table:2} and \cref{appendix:certified-ratio}. We remark that \method provides a probabilistic certification for point cloud models and we use a high confidence level of $99.9\%$ in all experiments; while the baseline DeepG3D~\cite{DeepG3D} yields a deterministic robustness certification.

For the part segmentation task, we evaluate our method using a fixed random subset of the ShapeNet test dataset. As the part segmentation task requires assigning a part category to each point in a point cloud, we report the \textbf{point-wise certified accuracy} defined as the fraction of points that are classified \emph{correctly} and \emph{consistently}. Note that other common metrics such as IoU can be easily derived based on our bound as well. 
We will focus on the point-wise certified accuracy for the convenience of comparison with baseline \cite{DeepG3D}.


\subsection{Main Results}

In this section, we present our main experimental results. Concretely, we show that: (1) the certified accuracy of \method under a range of semantic transformations is significantly higher than the baseline, and \method is able to certify under some transformation space where the baseline cannot be applied; (2) the certified accuracy of \method always outperforms the baseline for different point cloud sizes, and more interestingly, the certified accuracy of \method increases with the increasing of point cloud size while that of the baseline decreases due to relaxation; (3) \method is also capable of certifying against $\ell_2$ or $\ell_\infty$ norm bounded 3D perturbations for different point clouds sizes; (4) on the part segmentation task, \method still outperforms the baseline against different semantic transformations and is able to certify some transformation parameter space that the baseline is not applicable.


\begin{table}[tb]
    \centering
    \vspace{-0.1in}
    \caption{Comparison of certified accuracy achieved by our transformation-specific smoothing framework \method and the baseline, DeepG3D \cite{DeepG3D}. ``-" denotes the settings where the baselines cannot scale up to.}
    \vskip 0.1in
    \label{table:1}
    \resizebox{0.8\linewidth}{!}{
    \begin{small}
    \begin{tabular}{llcc}
    \toprule
    \multirow{2}{*}[-2pt]{Transformation} & \multirow{2}{*}[-2pt]{Attack radius} & \multicolumn{2}{c}{Certified Accuracy ($\%$)} \\[1pt] \cline{3-4}\\[-7pt]
    &           & \method &  DeepG3D  \\
    \midrule
    \multirow{2}{*}{ZYX-rotation} & $2^\circ$ & \textbf{81.4} & 61.6\\
    & $5^\circ$ & \textbf{69.2} & 49.6 \\
    \midrule
    \multirow{3}{*}{General rotation} & $5^\circ$ & \textbf{78.5} & -\\
    & $10^\circ$ & \textbf{69.2} & - \\
    & $15^\circ$ & \textbf{55.5} & - \\ 
    \midrule
    \multirow{3}{*}{Z-rotation} & $20^\circ$ & \textbf{84.2}  & 81.8 \\
    & $60^\circ$ & \textbf{83.8}  & 81.0\\
    & $180^\circ$& \textbf{81.3} & - \\
    \midrule
    \multirow{3}{*}{Z-shear} & 0.03 & \textbf{83.4} &  59.8\\
    &  0.1  & \textbf{82.2} & - \\
    & 0.2 & \textbf{77.7} & - \\ 
    \midrule
    \multirow{3}{*}{Z-twist} & $20^\circ$ & \textbf{83.8} & 20.3\\
    & $60^\circ$ & \textbf{80.1} & - \\
    & $180^\circ$ & \textbf{64.3} & - \\
    \midrule
    \multirow{3}{*}{Z-taper} & 0.1 & \textbf{78.1} & 69.0  \\
    & 0.2 & \textbf{76.5} & 23.9 \\
    & 0.5 & \textbf{66.0} & - \\ 
    \midrule
    \multirow{2}{*}{Linear} & 0.1 & \textbf{74.0} &- \\
    & 0.2 & \textbf{59.9} & -\\
    \midrule
    \multirow{3}{*}{\makecell[l]{Z-twist $\circ$\\Z-rotation}} & $20^\circ$, $1^\circ$& \textbf{78.9} & 13.8 \\
    & $20^\circ$, $5^\circ$& \textbf{78.5} & - \\
    & $50^\circ$, $5^\circ$& \textbf{76.9} & - \\ 
    \midrule
    \multirow{2}{*}{\makecell[l]{Z-taper $\circ$\\Z-rotation}}
    & 0.1, $1^\circ$ & \textbf{76.1} & 58.2 \\
    & 0.2, $1^\circ$ & \textbf{72.9} & 17.5\\ 
    \midrule
    \multirow{2}{*}{\makecell[l]{Z-twist $\circ$ Z-taper\\$\circ$ Z-rotation}} & $10^\circ$, 0.1, $1^\circ$& \textbf{68.8} & 17.5\\
    &$20^\circ$, 0.2, $1^\circ$ & \textbf{63.1} & 4.6 \\ 
    \bottomrule
    \end{tabular}
    \end{small}
    }
    \vskip -0.15in
\end{table}

\begin{table}[t]
\vspace{-0.1in}
\caption{Comparison of certified ratio as well as certified accuracy for z-rotation transformations. ``-" denotes the settings which the baselines cannot scale up to.}
\vskip 0.2in
    \label{table:2}
    \centering
    \begin{small}
    \begin{tabular}{lccccc}
    \toprule
        \multirow{2}{*}{Radius} & \multicolumn{2}{c}{Certified Ratio ($\%$)} & & \multicolumn{2}{c}{Certified Accuracy ($\%$)} \\
        \cline{2-3} \cline{5-6}\\[-7pt]
        & \method & DeepG3D & & \method & DeepG3D \\
        \midrule
        20$^\circ$ & \textbf{99.0} & 96.7 & & \textbf{84.2} & 81.8 \\
        60$^\circ$ & \textbf{98.1} & 95.7 & & \textbf{83.8} & 81.0 \\
        180$^\circ$ & \textbf{95.2} & - & & \textbf{81.3} & - \\
        \bottomrule
    \end{tabular}
    \end{small}
\vskip -0.2in
\end{table}
\subsubsection{Comparison of Certified Accuracy} 

\cref{table:1} shows the certified accuracy we achieved for different transformations compared with prior works. We train a PointNet model with 64 points, which is consistent with the baseline. For transformations characterized by one parameter, such as z-rotation, z-twist, and z-taper, we report the certified accuracy against attacks in $\pm \theta$. For z-shear with a parameter space of $\mathbb R^2$, we report the certified accuracy against attacks in a certain $\ell_2$ parameter radius. For the linear transformation with a parameter space of $\mathbb R^{3\times 3}$, we report the certified accuracy against attacks in a certain Frobenius norm radius.

The highlighted results in \cref{table:1} demonstrate that our framework \method significantly outperforms the state of the art in every known semantic transformation. For example, we improve the certified accuracy from $59.8\%$ to $83.4\%$ for z-shear in $\pm 0.03$ and from $20.3\%$ to $83.8\%$ for z-twist in $\pm 20^\circ$. 

Besides, we also report the certified accuracy for larger attack radius for which the baseline cannot certify (cells with ``-"). For instance, we achieve $81.3\%$ certified accuracy on z-rotation within $\pm 180^\circ$, which is essentially every possible z-rotation transformation. 

The general rotation transformations we define in \cref{general-rotation-section} includes rotations along any axis with bounded angles. \citet{DeepG3D} consider \emph{ZYX-rotation}, the composition of three rotations within $\pm \theta$ (Euler angles) along $x,y,z$ axes instead \yrcite{DeepG3D}, which results in a different geometric shape for the certified parameter space. However,  the parameter space restricted by $S^2\times [0,2\theta]$ of general rotation strictly contains the space defined by $\pm \theta$ for three Euler angles. (See \cref{appendix:zyx} for proof.) The derived results for ZYX-rotation are also shown in \cref{table:1} for comparison.

\textbf{Comparison of Certified Ratio.}
Aside from the certified accuracy, we also consider the certified ratio as another metric according to the baseline.  
This metric measures the tightness of certification bounds but fails to take the classification accuracy into account which is important. Therefore, we mainly present the comparison based on the certified ratio  for z-rotations in \cref{table:2} only for comparison and leave the full comparison in \cref{appendix:certified-ratio}.

\subsubsection{Certification on Point Clouds with different Sizes}

Here we show that our certification framework  naturally scales up to larger point cloud models. A basic principle of our \method framework is that deriving the certification bound for a smoothed classifier only depends on the predicted class probability. In other words, it does not rely on specific model architectures.

The relaxation-based verifiers \cite{DeepG3D, DeepPoly} have worse certification guarantees for larger point clouds due to the precision loss during relaxation, especially for pooling layers that are heavily used in point cloud model architectures. For example, the DeepG3D verifier guarantees $79.1\%$ certified accuracy for a 64-point model on z-rotation with $\pm 3^\circ$ (without splitting); but the certification drops to $32.3\%$ for a 1024-point model \cite{DeepG3D}.
In contrast, using our \method framework, the certified accuracy tends to \textbf{increase} with a larger number of points in point clouds. This is because larger PointNet models predict more accurately and yield higher class probability after smoothing. We compare our \method framework with the baseline in terms of certified accuracy for different point cloud sizes in \cref{table:3a}. The baseline DeepG3D only presents results for z-rotations in $\theta = \pm 3^\circ$, which cannot fully illustrate the capability of our method. Therefore, we also report our experimental results for z-rotations in $\theta = \pm 180^\circ$ in \cref{table:3b}. It shows that our method can scale up to larger point cloud models to accommodate real-world scenarios.

\begin{table}[tb]
\vspace{-0.1in}
    \caption{Certification of z-rotation for different point cloud sizes. The certified accuracy achieved by our \method increases as the size of the point cloud model increases.}
    \label{table:3}
    \vskip 0.1in
    \centering
    \begin{small}
    \begin{subtable}{\linewidth}
    \caption{$\theta=\pm 3^\circ$ compared with DeepG3D \cite{DeepG3D} }
    \label{table:3a}
    \centering
    \resizebox{0.9\linewidth}{!}{
    \begin{tabular}{lccccccc}
    \toprule
        Points &16 & 32 & 64 & 128 & 256 & 512 & 1024\\  
        \hline\\[-7pt]
        \method & \textbf{83.2} & \textbf{83.8} & \textbf{86.6} & \textbf{87.4} & \textbf{89.4} & \textbf{89.8} & \textbf{90.5} \\
        DeepG3D & 75.4 & 78.4 & 79.1 & 69.4 & 57.5 & 42.8 & 32.3 \\
        \bottomrule
    \end{tabular}
    }
    \vspace{.05in}
    \end{subtable}
    \begin{subtable}{\linewidth}
    \caption{Certified accuracy of \method under $\theta=\pm 180^\circ$}
    \centering
    \label{table:3b}
    \resizebox{0.9\linewidth}{!}{
    \begin{tabular}{lccccccc}
    \toprule
        Points &16 & 32 & 64 & 128 & 256 & 512 & 1024\\  
        \hline\\[-7pt]
        \method & 73.6 & 79.3 & 81.3 & 81.8 & 83.0 & 84.6 & 83.8\\
        \bottomrule
    \end{tabular}
    }
    \end{subtable}
    \end{small}
    \vskip -0.2in
\end{table}

\subsubsection{Certification against $\ell_p$ norm Bounded 3D-Perturbations}

In addition to semantic transformations, we also provide robustness certification for point cloud models against $\ell_p$ perturbations. Defenses have been proposed for point cloud models against $\ell_2$ norm bounded perturbations \cite{Marc2021}, which is a special case of our \method framework regarding an additive transformation $\phi(x,z) = x+z$.


We cannot directly certify against perturbations with bounded $\ell_\infty$ norm. However, a certification bound similar to the baseline \cite{DeepG3D} can still be derived, using the loose inequality $\|\theta\|_\infty \leq \sqrt{3N}\|\theta\|_2$. Here, we exhibit the certified accuracy for bounded $\ell_2$ norm in \cref{table:5} as a benchmark; and omit more details for $\ell_\infty$ norm to \cref{appendix:ellp}. We show that under the $\ell_2$ norm bounded 3D-perturbations, \method certifies even better as the point clouds sizes increase and achieve a high certified accuracy of $77.3\%$ for perturbations with $\ell_2$ norm bounded by 0.1.

\begin{table}[tb]
\vspace{-0.1in}
\centering
\caption{Certified accuracy of \method for point cloud models under $\ell_2$ attacks. The certified accuracy increases as the size of point cloud models increases.}
\vskip 0.15in
\label{table:5}
\resizebox{0.7\linewidth}{!}{
\begin{small}
    \begin{tabular}{ccccc}
    \toprule
    \multirow{2}{*}{Attack} & \multirow{2}{*}{Radius} & \multicolumn{3}{c}{Certified Accuracy ($\%$)}  \\
    \cline{3-5}\\[-7pt]
    & & 16 & 64 & 256 \\
    \midrule
    $\ell_2$ & 0.05 & 74.1 & 82.2 & {84.2} \\
    $\ell_2$ & 0.1 & 61.9 & 70.8 & {77.3} \\
    \bottomrule
    \end{tabular}
    \end{small}
}
\vskip -0.2in
\end{table}

\subsubsection{Certification for Part Segmentation}
\label{part-segmentation}
Part segmentation is a common 3D recognition task in which a model is in charge of assigning each point or face of a 3D mesh to one of the predefined categories. As our \method framework is independent of concrete model architectures, it can be naturally extended to handle this task.

We evaluate our method using the ShapeNet part dataset \cite{shapenet-dataset}. We train a segmentation version PointNet \cite{pointnet} with 64 points, which predicts a part category for each point in the point cloud. The certified accuracy reported in \cref{table:6} denotes the percentage of points guaranteed to be labeled correctly.
We can see that for the part segmentation task, \method consistently outperforms the baseline against different semantic transformations. The baseline only reports the result for z-rotations in $\pm 5^\circ$ and $\pm 10^\circ$, while we present robustness guarantees for any z-rotation ($\pm 180^\circ$) as well as other transformations including shearing and twisting.

\begin{table}[tb]
    \caption{Comparison of point-wise certified accuracy for the \textit{part segmentation} task. ``-" denotes the settings that the baseline does not consider or cannot scale up to. }
    \label{table:6}
    \vskip 0.15in
    \centering
    \resizebox{0.8\linewidth}{!}{
    \begin{small}
    \begin{tabular}{lccc}
    \toprule
        \multirow{2}{*}[-2pt]{Transformation} & \multirow{2}{*}[-2pt]{Radius} & \multicolumn{2}{c}{Certified Accuracy ($\%$)}\\
        \cline{3-4}\\[-7pt]
        & & \method & DeepG3D \\
    \midrule
        Z-rotation & 5$^\circ$& \textbf{87.8} & 85.7\\
        Z-rotation & 10$^\circ$ & \textbf{86.1} & 84.8\\
        Z-rotation & 180$^\circ$ & \textbf{70.8} & -\\
        Z-shear & 0.2 & \textbf{86.1} & - \\
        Z-twist & 180$^\circ$ & \textbf{74.5} & -\\
    \bottomrule
    \end{tabular}
    \end{small}
    }
    \vskip -0.2in
\end{table}

\section{Conclusion}
In this work, we propose a unified certification framework \method for point cloud models against a diverse range of semantic transformations. Our theoretical and empirical analysis show that \method is more scalable and able to provide much tighter certification under different settings and tasks. 

\section*{Acknowledgements}
The authors thank the anonymous reviewers for their valuable feedback. 
This work is partially supported by 
 NSF grant No.1910100, NSF CNS No.2046726, C3 AI, and the Alfred P. Sloan Foundation.
WC would like to thank the support from Institute for Interdisciplinary Information Sciences, Tsinghua University.

\bibliography{reference}
\bibliographystyle{icml2022}

\newpage
\appendix
\onecolumn

\section{Generic Theorems for Composable and Indirectly Composable Transformations}
\subsection{Main Theorem for Transformation Specific Smoothing}
\label{appendix:generic}
\begin{theorem}[Theorem 1 \cite{TSS}]
\label{thm:generic}
Let $\epsilon_0\sim \mathbb P_0$ and $\epsilon_1 \sim \mathbb P_1$ be $\mathcal Z$-valued random variables with probability density function $f_0$ and $f_1$. Let $\phi:\mathcal X\times \mathcal Z\to \mathcal X$ be a semantic transformation. Suppose a classifier smoothed by the transformation $\phi$ predicts $y_A = g(x;\epsilon)$, and that
\begin{equation}
    q(y_A|x,\epsilon) \geq p_A > p_B \geq \max_{y\neq y_A} q(y|x,\epsilon).
    \label{class_probability}
\end{equation}
For $t\geq 0$, we define sets $\underline S_t, \overline S_t\subseteq Z$ by $\underline S_t := \{f_1/f_0 < t\}$ and $\overline S_t := \{f_1/f_0 \leq t\}$. Also define a function $\xi: [0,1]\to [0,1]$ by
\begin{align}
    \xi(p) &:= \sup\{\mathbb P_1(S) : \underline S_{\tau_p} \subseteq S\subseteq \overline S_{\tau_p}\}\\
    \text{where}\ &\tau_p := \inf\{t\geq 0: \mathbb P_0(\overline S_t)\geq p\}.
\end{align}
Then it is guaranteed that $g(x;\epsilon_1) = g(x;\epsilon_0)$ if the following condition holds:
\begin{equation}
\label{eq:generic-condition}
    \xi(p_A) + \xi(1-p_B) > 1.
\end{equation}
\end{theorem}

Intuitively, $\xi(p_A)$ computes a lower bound for the class probability of $y_A$ when the smoothing distribution changes from $\epsilon_0$ to $\epsilon_1$. Suppose we want to certify an $\epsilon_0$-smoothed classifier against an \resolve transformation $\phi$ and $\phi(\phi(x,\alpha),\beta) = \phi(x,\gamma_\alpha(\beta))$. For any attack $\alpha\in \mathcal Z$, we assign $\epsilon_1 = \gamma_\alpha(\epsilon_0)$ and check for the condition of \cref{eq:generic-condition}. Moreover, if $\phi$ is additive and $\epsilon_0$ is a Gaussian random variable, we have the following corollary:
\begin{corollary}[Corollary 7 \cite{TSS}]
\label{additive+gaussian}
Let $\phi:\mathcal X \times \mathcal Z \to \mathcal X$ be an additive transformation and $\mathcal Z = \mathbb R^m$. Suppose classifier $g$ is smoothed by a random variable $\epsilon_0\sim \mathcal N(0,\diag(\sigma_1^2, \dots, \sigma_m^2))$. Assume that the class probability satisfies:
\begin{equation}
    q(y_A|x,\epsilon_0) \geq p_A > p_B \geq \max_{y\neq y_A}(y|x,\epsilon_0).
\end{equation}
Then it holds that $g(x;\epsilon_0) = g(\phi(x,\alpha);\epsilon_0)$ if the attack parameter $\alpha$ satisfies:
\begin{equation}
    \sqrt{\sum_{i=1}^m \left(\frac{\alpha_i}{\sigma_i}\right)^2} < \frac{1}{2}\Big(\Phi^{-1}(p_A) - \Phi^{-1}(p_B)\Big).
\end{equation}
\end{corollary}
We direct readers to \cite{TSS} for the rigorous proof on \cref{thm:generic} and \cref{additive+gaussian} .

\subsection{Theorem for Certifying Indirectly Composable Transformations}
\label{appendix:TSS-DR}
\begin{theorem}[Corollary 2 \cite{TSS}]
\label{thm:TSS-DR}
Let $\psi(x,\delta) = x + \delta$ and $\epsilon \sim \mathcal N(0,\sigma^2\mathds 1_d)$. $\phi:\mathbb X \times \mathcal Z_\phi \to \mathcal X$ is a \difresolve transformation. Construct a smoothed classifier with additive noise $\psi(x,\delta)$ and suppose it predicts $y = \arg\max_{y\in \mathcal Y} q(y|x;\epsilon)$. Draw $N$ samples $\{\alpha_i\}_i^N$ from a set $\mathcal S \subseteq Z_\phi$. Assume 
\begin{equation}
    q(y_A|\phi(x,\alpha_i),\epsilon)\geq p_A^{(i)} \geq p_B^{(i)} \geq \max_{y\neq y_A}q(y|\phi(x,\alpha_i),\epsilon).
\end{equation}

Then it is guaranteed that $\forall \alpha\in \mathcal S: y_A = \arg\max_{y\in \mathcal Y} q(y|\phi(x,\alpha);\epsilon)$ if the maximum interpolation error
\begin{align}
\label{eq:interpolation}
    \mathcal M_{\mathcal S} &:= \max_{\alpha\in \mathcal S}\min_{1\leq i\leq N}\mathcal M(\alpha, \alpha_i)\\
    &= \max_{\alpha\in \mathcal S}\min_{1\leq i\leq N}\|\phi(x,\alpha)-\phi(x,\alpha_i)\|_2\\
        \text{satisfies}\ &\mathcal M_{\mathcal S} < R := \frac{\sigma}{2}\min_{1\leq i\leq N} \Big(\Phi^{-1}(p_A^{(i)}) - \Phi^{-1}(p_B^{(i)})\Big)
\end{align}
\end{theorem}

\section{Proofs for Certifying Specific Transformations}

Here, we present proofs for the theorems and corollaries proposed in \cref{section:4}, including concrete protocols for common 3D transformations, such as z-rotation, z-twist, z-taper, etc.

\subsection{Proof of \cref{thm:taper}: Certifying Z-taper}
\label{appendix:taper}
\emph{Proof.} The z-taper transformation is defined as $\phi_{TP}:\mathcal X \times \mathcal \mathbb R \to \mathcal X$ where $\mathcal X = \mathbb R^{3\times N}$ is the space for input point clouds. A z-taper transformation acting on a point cloud $x = \{p_i\}_{i=1}^N\in \mathcal X$ in fact performs point-wise transformation to each $p_i$, where
\begin{equation}
    \phi_{TP}(p_i,\theta) = \begin{pmatrix}
    x_i(1+\theta z_i)\\
    y_i(1+\theta z_i)\\
    z_i
    \end{pmatrix},\ \text{if}\ p_i = (x_i,y_i,z_i)^T.
\end{equation}
We calculate the interpolation error between two parameters $\theta$ and $\theta_j$ by
\begin{align}
    \mathcal M(\theta, \theta_j) &= \|\phi_{TP}(x,\theta) - \phi_{TP}(x,\theta_j)\|_2\\
    &= \left(\sum_{i=1}^N \|\phi_{TP}(p_i,\theta) - \phi_{TP}(p_i,\theta_j)\|_2^2\right)^{1/2}\\
    &= \left(\sum_{i=1}^N (x_i^2 + y_i^2) z_i^2(\theta-\theta_j)^2\right)^{1/2}\\
    &\leq \frac{\sqrt{N}|\theta - \theta_j|}{2}.
\end{align}
The last inequality holds because we assume point clouds are normalized into a unit ball, so $(x_i^2 + y_i^2) z_i^2 \leq \frac{1}{4}$. Recall that we choose $\theta_j = (\frac{2j}{M} - 1)R$ and $j = 0,1,\dots, M$. Hence, $
    \max_{\theta\in [-R,R]}\min_{j} |\theta - \theta_j| < \frac{R}{M}.$
The maximal interpolation error is thus bounded by
\begin{align}
    \mathcal M_S &= \max_{\theta\in [-R,R]}\min_j \mathcal M(\theta, \theta_j)\\
    &\leq \frac{R\sqrt{N}}{2M}.
\end{align}
According to \cref{thm:TSS-DR}, it is guaranteed for all $\theta \in [-R,R]$ that $y_A = \arg\max_y q(y|\phi_{TP}(x,\theta);\epsilon)$, if $\ \forall j$,
\begin{equation}
    \frac{\sigma}{2}\Big(\Phi^{-1}\Big(p_A^{(j)}\Big) - \Phi^{-1}\Big(p_B^{(j)}\Big) \Big) \geq \frac{R\sqrt{N}}{2M}.
\end{equation}

\subsection{Proof of \cref{thm:general_rotation}: Certifying General Rotation}
\label{appendix:general_rot}
\emph{Proof.} We first recall the definition of general rotation: $\phi_R:\mathcal X \times \mathcal Z\to \mathcal X$. The parameter space $\mathcal Z = S^2 \times \mathbb R^+$ where $S^2$ characterizes the rotation axis and $\mathbb R^+$ stands for the rotation angle. By Euler's theorem, general rotations are \resolve transformations; and the composition of two rotations can be expressed by:
\begin{align}
    &\phi_R(\phi_R(x,z_1),z_2) = \phi_R(x,z_3)\\
    \text{where } &\begin{cases}
    k_3 = \text{normalize}(\sin\frac{\theta_1}{2}\cos\frac{\theta_2}{2}k_1 + \cos\frac{\theta_1}{2}\sin\frac{\theta_2}{2}k_2 + \sin\frac{\theta_1}{2}\sin\frac{\theta_2}{2}k_2 \times k_1)\\
    \theta_3 = 2\arccos{(\cos\frac{\theta_1}{2}\cos\frac{\theta_2}{2}-\sin\frac{\theta_1}{2}\sin\frac{\theta_2}{2}k_1\cdot k_2)}
    \end{cases}
\end{align}
The interpolation error of a point cloud $x = \{p_i\}_{i=1}^N$ between two transformations with parameters $z = (k,\theta)$ and $z_j = (k_j,\theta_j)$ is bounded by:
\begin{align}
    \mathcal M(z,z_j) &= \|\phi_R(x,z) - \phi_R(x,z_j)\|_2\\
    &= \|\phi_R(\phi_R(x,z),z_j^{-1}) - x\|_2, \quad (\text{where}\ z_j^{-1} = (-k_j, \theta_j))\\
    &= \left(\sum_{i=1}^N \|\phi_R(\phi_R(p_i,z), z_j^{-1}) - p_i\|_2^2 \right)^{1/2}\\
    &= \left(\sum_{i=1}^N \|\phi_R(p_i, z^\prime) - p_i\|_2^2\right)^{1/2}\quad (\text{Let}\ z^\prime = (k^\prime, \theta^\prime) \ \text{be the composition of}\ z,z_j^{-1}) \\
    &\leq \left(\sum_{i=1}^N (\theta^{\prime} \|p_i\|_2 )^2\right)^{1/2} = \theta^{\prime} \|x\|_2. \label{eq:45}
\end{align}

Assuming $\langle k,k_i\rangle\leq \epsilon$, $|\theta-\theta_i| \leq \delta$, we derive that for $\theta,\theta_i\in [0,R]$,
\begin{align}
    \cos \frac{\theta^\prime}{2} &= \cos \frac{\theta_j}{2}\cos \frac{\theta}{2} + \cos \langle k,k_i\rangle \sin \frac{\theta_j}{2}\sin \frac{\theta}{2}, &&(\langle k,k_i\rangle \leq \epsilon)\\
    &\geq \cos \frac{\theta_j - \theta}{2} - \frac{\epsilon^2}{2}\sin \frac{\theta_j}{2}\sin \frac{\theta}{2} &&(\text{since}\ \cos \epsilon \geq 1 - \frac{\epsilon^2}{2})\\
    &\geq 1 - \left(\frac{\theta_j - \theta}{2}\right)^2 - \frac{\epsilon^2 \theta\theta_j}{8}.&& (\text{since}\ \sin x \leq x)\\
    &\geq 1 - \frac{\delta^2}{4} - \frac{\epsilon^2 R^2}{8}.
\end{align}
Note that $\arccos(1-x) \leq \frac{\pi}{2}\sqrt{x}$ when $x\in [0,1]$, we have
\begin{align}
    \theta^\prime &\leq 2\arccos\left(1 - \frac{\delta^2}{4} - \frac{\epsilon^2 R^2}{8}\right)\\
    &\leq \pi \sqrt{\frac{\delta^2}{4}+\frac{\epsilon^2 R^2}{8}}.\label{eq:51}
\end{align}
Combining \cref{eq:45} and \cref{eq:51}, the maximal interpolation error for $z\in S^2\times [0,R]$ satisfies
\begin{align}
\label{eq:inter_err_rot}
    \mathcal M_S &= \max_{z}\min_{j}\mathcal M(z,z_j)\\
    &\leq \pi \sqrt{\frac{\delta^2}{4}+\frac{\epsilon^2 R^2}{8}} \|x\|_2.
\end{align}
\cref{thm:general_rotation} thus holds combining \cref{eq:inter_err_rot} with \cref{thm:generic}.

Moreover, we specify a sampling strategy to satisfy the condition of \cref{eq:rotation_condition}.
\begin{itemize}
\setlength{\parskip}{0pt}
    \item Uniformly sample $\pi M$ number of $a_r\in [0,\pi]$.
    \item For each $a_r$, uniformly sample $2\pi M \sin\theta_s$ points $b_{rs}\in [0,2\pi]$.
    \item Uniformly sample $M$ number of $\theta_t \in [0,R]$.
    \item Draw $O(M^3)$ samples in total: $z_j = (k_j,\theta_t)$ with $k_j = (\cos b_{rs}\sin a_r, \sin b_{rs}\sin a_r, \cos a_r)$.
\end{itemize}
Following this strategy, the sampled parameters distribute evenly in the subspace of $S^2 \times [0,R]$, which guarantees $\epsilon = \frac{\sqrt 2}{2M}$ and $\delta = \frac{R}{2M}$ for the conditions in \cref{eq:rotation_condition}.

To sum up, it is guaranteed that for all $z\in S^2 \times [0,R], x\in \mathcal X: y_A = \arg\max_y q(y|\phi_R(x,z);\epsilon)$, if $\forall j$,
\begin{equation}
    \frac{\sigma}{2}\Big(\Phi^{-1}\Big(p_A^{(j)}\Big) - \Phi^{-1}\Big(p_B^{(j)}\Big) \Big) \geq \frac{\sqrt{2}\pi R \|x\|_2}{4M}.
\end{equation}

\begin{remark}
In practise, we implement a tighter bound that $\arccos(1-x) \leq \sqrt{2x} + (\frac{\pi}{2}-\sqrt{2})x^{\frac{3}{2}}$ for $x\in [0,1]$.
\end{remark}

\subsection{From General Rotation to ZYX-rotation}
\label{appendix:zyx}
ZYX-rotation, the composition of three rotations along $x,y$ and $z$ axes, is defined by: $\phi_{ZYX-rot}:\mathcal X\times \mathcal Z \to \mathcal X$ with parameter space $\mathcal Z = \mathbb R^3$. Specifically, for $z = (\alpha, \beta, \gamma)\in \mathcal Z$ and $x = \{p_i\}_{i=1}^N\in \mathcal X$,
\begin{equation}
    \phi_{ZYX-rot}(p_i, z) = R_z(\gamma) R_y(\beta) R_x(\alpha) p_i,\ \text{where } R_z, R_y, R_x\ \text{are the rotation matrix along x,y,z axes.}
\end{equation}
Note that the rotation angle for any rotation matrix $R$ can be calculated by:
\begin{equation}
    |\theta| = \arccos\left(\frac{tr(R)-1}{2}\right)
\end{equation}
The trace of the rotation matrix for ZYX-rotation is
\begin{equation}
    f(\alpha,\beta,\gamma) = tr(R_z(\gamma)R_y(\beta)R_x(\alpha)) = \cos \alpha\cos\beta + \cos \alpha\cos\gamma + \cos\beta\cos\gamma - \sin\alpha\sin\beta\sin\gamma.
\end{equation}

We assume $\alpha,\beta,\gamma \in [-\frac{\pi}{2},\frac{\pi}{2}]$. $\frac{\partial f}{\partial \alpha} = \frac{\partial f}{\partial \beta} = \frac{\partial f}{\partial \gamma} = 0$ yields $\alpha = \beta = \gamma = 0, \pm\frac{\pi}{2}$. Therefore, for $\alpha, \beta, \gamma \in [-\varphi, \varphi]$ with $\varphi \in [0,\frac{\pi}{2}]$, the minimum of $f(\alpha, \beta,\gamma)$ can only be on $\alpha, \beta, \gamma = \pm \varphi$ or $\alpha = \beta = \gamma = 0$. Since $\alpha = \beta = \gamma = 0$ yields the maximum $f(\alpha, \beta,\gamma)$, we have
\begin{equation}
    \min_{\alpha,\beta,\gamma\in [-\varphi,\varphi]} f(\alpha,\beta,\gamma) = 3\cos^2\varphi - \sin^3\varphi
\end{equation}
Thus,
\begin{align}
    \cos\theta &= \frac{tr(R)-1}{2} \\
    &\geq \frac{3\cos^2\varphi - \sin^3 \varphi - 1}{2}\\
    &= \frac{(2\cos^2 \varphi - 2\sin^2 \varphi )+ (\sin^2 \varphi - \sin^3\varphi)}{2}\\
    &\geq \cos 2\varphi.
\end{align}
The rotation angle $\theta$ is thus bounded by $\theta \leq 2\varphi$. Hence, any transformation $\phi_{ZYX-rot}$ with $z\in [-\theta, \theta]^3$ and $\theta \in [0,\pi/2]$ belongs to the set of general rotations $\phi_R$ with parameter space $\mathcal Z_R = S^2 \times [0,2\theta]$.

\subsection{Proof of \cref{thm:linear}: Certifying Linear Transformations}
\begin{proof}
The set of linear transformations is defined by $\phi_L: \mathcal X \times \mathcal Z \to \mathcal X$ where the parameter space is $\mathcal Z = \mathbb R^{3\times 3}$ and $\phi_L(p_i, A) = (I + A)p_i$. The composition of two linear transformations can be expressed by
\begin{align}
    \phi(B)\phi(A) &= I + A + B + BA\\
    &= I + A + (b_{ij} + \sum_{k} b_{ik}a_{kj})_{ij}
\end{align}
To certify these transformations, we smooth the classifier with $\phi(E)$, where $e_{ij} \sim \mathcal N(0,\sigma^2)$. The smoothed classifier is denoted as $g(x; E)$. Then the smoothed classifier can be viewed as smoothed additively by an equivalent random variable $\tilde E$.

\begin{align}
    &g(\phi(A, x); E) = g(x; A+\tilde E)
\end{align}

Suppose $E \sim \mathcal N (0, \sigma^2 \mathcal I_{9})$, then

\begin{equation}
\label{eq:tildeE}
    \tilde E =
    \begin{pmatrix}
    I + A^T & O & O \\
    O & I + A^T & O \\
    O & O & I + A^T
    \end{pmatrix} E = SE.
\end{equation}
The covariance matrix $\Sigma$ of $\tilde E$ is thus $\Sigma = SS^T$. This symmetric covariance matrix can be decomposed by $\Sigma = QDQ^T$. Suppose the singular values of $\Sigma$ are $k_i$. Then by the Mirsky Theorem for matrix perturbation, we have
\begin{equation}
    \sqrt{\Sigma_{i=1}^9 (k_i - 1)^2} \leq \|A\|_F
\end{equation}

Therefore, the problem is reduced to the diagonal case such that the covariance matrix is diagonal. Assume $\epsilon_0 \sim \mathcal (0, \sigma^2 I_{9})$ and $\epsilon_1 \sim \mathcal (0, \sigma^2 \cdot \mathrm {diag}(k_1^2, k_2^2, \dots, k_9^2))$. Let $A = \mathrm{diag} (k_1, k_2,\dots, k_9)$ and $\Sigma = \sigma^2 I_9$. Also, we assume without loss of generality by symmetry that $k_1 = k_2 = k_3, k_4 = k_5 = k_6$ and $k_7 =k_8 = k_9$.

Therefore, according to \cref{additive+gaussian}, the smoothed classifier is guaranteed to be robust under attack $A$, if
\begin{equation}
    \sqrt{\sum_{i=1}^9 \Big(\frac{a_{i}}{k_i}\Big)^2}\leq \frac{\sigma}{2}\Big(\Phi^{-1}(\tilde p_A) - \Phi^{-1}(\tilde p_B)\Big)
\end{equation}
where $\tilde p_A = q(y|x;\tilde E)$ given $p_A = q(y|x;E)$. Since $k_i \geq 1 - \|A\|_F,\forall i$, the condition can be simplified to
\begin{equation}
    \sqrt{\sum_{i=1}^9 a_i^2} = \|A\|_F \leq \frac{\sigma(1 - \|A\|_F)}{2}\Big(\Phi^{-1}(\tilde p_A) - \Phi^{-1}(\tilde p_B)\Big)
\end{equation}
which is equivalent to the following condition when $\|A\|_F \leq 1$.
\begin{equation}
    \|A\|_F \leq \frac{\sigma \Big(\Phi^{-1}(\tilde p_A) - \Phi^{-1}(\tilde p_B)\Big) }{2 + \sigma\Big(\Phi^{-1}(\tilde p_A) - \Phi^{-1}(\tilde p_B)\Big)}.
\end{equation}
\end{proof}

Next we find a functional relation between $\tilde p_A$ and $p_A$. To do so, we first state and prove a lemma.

\begin{lemma}
\label{lemma:B1}
Suppose a classifier $g$ is smoothed by a Gaussian random variable $\epsilon_0 \sim \mathcal N(0,\mathrm{diag}(\sigma_1^2,\dots, \sigma_n^2))$, while another classifier $g^\prime$ is smoothed by a Gaussian random variable $\epsilon_1\sim \mathcal N(0,\mathrm{diag}(k^2\sigma_1,\dots, k^2\sigma_{m}^2, \sigma_{m+1}^2,\dots, \sigma_{n}^2))$. If the smoothed classifier $g$ predicts $p_A = q(y_A|x;\epsilon_0)$, then
\begin{equation}
\label{chi-square}
    p_A^\prime = q(y|x;\epsilon_1) \geq \begin{cases}
    F_{\chi^2_m}(\frac{1}{k^2}F^{-1}_{\chi^2_m} (p_A))\quad (k\geq 1)\\
    1 - F_{\chi^2_m}(\frac{1}{k^2} F^{-1}_{\chi^2_m}(1-p_A)) \quad (k<1).
    \end{cases}
\end{equation}
\end{lemma}
\begin{proof}
We denote by $f_0$ and $f_1$ the probability density function of $\epsilon_0$ and $\epsilon_1$. We define a level function as $\Lambda(z) := \frac{f_1(z)}{f_0(z)}$. Suppose $\epsilon_0\sim \mathcal N(0, \Sigma)$ and $\epsilon_1 \sim \mathcal N(0,A^2\Sigma)$. Therefore, $A$ is a diagonal matrix with $A_{ii} = k$ if $i \leq m$ and $A_{ii} = 1$ otherwise.
\begin{align}
    \Lambda(z) &= \frac{f_1(z)}{f_0(z)}\\[0.5em]
    &= \frac{\left((2\pi)^n |A^2 \Sigma |\right)^{-1/2} \exp (-\frac{1}{2} (z^T (A^2 \Sigma)^{-1} z)) }{\left((2\pi)^n |\Sigma|\right)^{-1/2} \exp (-\frac{1}{2} z^T(\Sigma)^{-1} z)}\\[0.5em]
    &= \frac{\left((2\pi)^n \sigma^{2n} k^{2m}\right)^{-1/2} \exp (-\frac{1}{2} \sum_{i=1}^{m} \frac{z_i^2}{k^2\sigma_i^2})}{\left((2\pi)^n \sigma^{2n}\right)^{-1/2} \exp (-\frac{1}{2} \sum_{i=1}^{m}\frac{z_i^2}{\sigma_i^2})}\\[0.5em]
    &= \frac{1}{k^{2m}} \exp \left( \sum_{i=1}^{m}\frac{z_i^2}{2\sigma_i^2} \left(1-\frac{1}{k^2}\right)\right)
\end{align}
We define lower level sets as $S_t := \{z\in \mathcal Z: \Lambda(z)\leq t\}$. When $k > 1$, we can write the series of lowe level sets as
\begin{equation}
    S_t = \left\{z\ \Big|\ \sum_{i=1}^m \frac{z_i^2}{\sigma^2} \leq t\right\}, t \geq 0,
\end{equation}

while for $k < 1$,
\begin{equation}
    S_t = \left\{z\ \Big|\ \sum_{i=1}^m \frac{z_i^2}{\sigma^2}\geq t\right\}, t\geq 0. 
\end{equation}
By Neyman Pearson lemma, we can lower bound $q(y|x;\epsilon_1)$ by
\begin{equation}
    q(y|x;\epsilon_1)\geq \mathbb P_1(S_{t^*}),\ \text{where}\ \mathbb P_0(S_{t^*}) = q(y|x;\epsilon_0).
\end{equation}

We denote by $F_{\chi_m^2}(\cdot)$ the cumulative density function of a chi-square distribution with $m$ degree of freedom. Then,

\begin{equation}
    \mathbb P_0\Big(\sum_{i=1}^{m} \frac{z_i^2}{\sigma_i^2} \leq t^\prime \Big) = F_{\chi^2_m}(t^\prime)
\end{equation}
And 
\begin{equation}
    \mathbb P_1\Big(\sum_{i=1}^{m} \frac{z_i^2}{\sigma_i^2} \leq t^\prime \Big) = F_{\chi^2_m}(\frac{t^\prime}{k^2})
\end{equation}

Thus, for $k > 1$,
\begin{align}
    q(y|x;\epsilon_1) &\geq \mathbb P_1(S_{t^\prime})\\
    &= F_{\chi_m^2}(\frac{1}{k^2}F_{\chi_m^2}^{-1}(p_A)).
\end{align}
For $k < 1$, however, we have $\mathbb P_0(S_t) = 1 - F_{\chi_m^2}(t^\prime)$, so 
\begin{align}
    q(y|x;\epsilon_1) &\geq \mathbb P_1(S_{t^\prime})\\
    & = 1 - F_{\chi_m^2}(\frac{1}{k^2}F_{\chi_m^2}^{-1}(1 - p_A)).
\end{align}
\end{proof}

Remember that our goal is to estimate $\tilde p_A = q(y|x;\tilde E)$. Since its covariance matrix $\Sigma$ only has three unique eigenvalues, leveraging \cref{lemma:B1} three times helps address the class probability $\tilde p_A$.

\begin{lemma}
\label{lemma:pAtilde}
Suppose a classifier $g$ is smoothed by random linear transformations with random variable $E\sim \mathcal N(0,\sigma^2 I^9)$. If $p_A = q(y_A|x,E)$ and $\tilde p_A = q(y_A|x;\tilde E)$ with $\tilde E \sim \mathcal N(0,\mathrm{diag}(k_1^2I_3, k_2^2I_3,k_3^2I_3)\sigma^2).$ If $(k_1-1)^2 + (k_2-1)^2 + (k_3-1)^2 \leq R$, then
\begin{equation}
    \tilde p_A \geq \min\{p_1,p_2,p_3,p_4\}.
\end{equation}
where
\begin{equation}
    p_1 =  F_{\chi^2_3}\left(\frac{1}{(1+\frac{R}{\sqrt{3}})^6}F_{\chi_3^2}^{-1}(p_A)\right).
\end{equation}
\begin{equation}
    p_2 =  1 - \frac{1}{(1-\frac{R}{\sqrt{3}})^6} F_{\chi_3^2}^{-1}(1 - p_A).
\end{equation}
\begin{equation}
    p_3 = \inf_{r_1^2 + r_2^2 = R^2, r_1,r_2\geq 0} F_{\chi_3^2}\left(\frac{1}{(1+\frac{r_1}{\sqrt{2}})^4}F_{\chi_3^2}^{-1}\left(1 - F_{\chi_3^2}\left(\frac{1}{(1-r_2)^2} F_{\chi_3^2}^{-1}(1-p_A)\right) \right) \right)
\end{equation}
\begin{equation}
    p_4 = \inf_{r_1^2 + r_2^2 = R^2, r_1,r_2\geq 0} 1 - F_{\chi_3^2}\left(\frac{1}{(1-\frac{r_1}{\sqrt{2}})^4}F_{\chi_3^2}^{-1}\left(1 - F_{\chi_3^2}\left(\frac{1}{(1+r_2)^2} F_{\chi_3^2}^{-1}(p_A)\right) \right) \right)
\end{equation}
\end{lemma}
\begin{proof}
The four probabilities $p_1,p_2,p_3,p_4$ corresponds to four different cases for $k_1,k_2,k_3$.

\begin{itemize}
\item $k_1,k_2,k_3 \geq 1$. In this case 
\begin{equation}
    F_{\chi_3^2}^{-1}(\tilde p_A) = \frac{1}{k_1^2k_2^2k_3^2}F_{\chi_3^2}^{-1}(p_A).
\end{equation}
Thus,
\begin{equation}
    \inf_{k_1,k_2,k_3\geq 1}\tilde p_A = F_{\chi^2_3}\left(\frac{1}{(1+\frac{R}{\sqrt{3}})^6}F_{\chi_3^2}^{-1}(p_A)\right).
\end{equation}
\item $k_1,k_2,k_3 < 1$. In this case
\begin{equation}
    F_{\chi_3^2}^{-1}(1- \tilde p_A) = \frac{1}{k_1^2k_2^2k_3^2}F_{\chi_3^2}^{-1}(1 - p_A) \leq \frac{1}{(1-\frac{R}{\sqrt{3}})^6} F_{\chi_3^2}^{-1}(1 - p_A).
\end{equation}
Thus,
\begin{equation}
    \inf_{k_1,k_2,k_3<1}\tilde p_A = 1 - F_{\chi_3^2}\left(\frac{1}{(1-\frac{R}{\sqrt{3}})^6}F_{\chi_3^2}^{-1}(1 - p_A)\right)
\end{equation}

\item $k_1,k_2 \geq 1, k_3 < 1$. Let $r_1 = \sqrt{(k_1-1)^2 + (k_2-2)^2}$ and $r_2 = 1 - k_3$.
\begin{equation}
    \tilde p_A = 
    F_{\chi_3^2}\Big(\frac{1}{k_1^2k_2^2}F_{\chi_3^2}^{-1}(1-p_A^\prime)\Big) \geq F_{\chi_{3}^2}\Big(\frac{1}{(1+\frac{r_1}{\sqrt 2})^4}F_{\chi_3^2}^{-1}(1 - p_A^\prime)\Big).
\end{equation}
where $p_A^\prime = 1 - F_{\chi_3^2}(\frac{1}{k_3^2} F_{\chi_3^2}^{-1}(1-p_A))$. Hence,
\begin{equation}
\inf_{k_1,k_2\geq 1,k_3<1}\tilde p_A = \inf_{r_1^2 + r_2^2 = R^2,r_1,r_2\geq 0} F_{\chi_3^2}\left(\frac{1}{(1+\frac{r_1}{\sqrt{2}})^4}F_{\chi_3^2}^{-1}\left(1 - F_{\chi_3^2}\left(\frac{1}{(1-r_2)^2} F_{\chi_3^2}^{-1}(1-p_A)\right) \right) \right).
\end{equation}
\item $k_1,k_2 < 1$ and $k_3\geq 1$. Let $r_1 = \sqrt{(1-k_1)^2 + (1-k_2)^2}$ and $r_2 = k_3 - 1$, so
\begin{equation}
    \tilde p_A = 1 - F_{\chi_3^2}\Big(\frac{1}{k_1^2k_2^2}F_{\chi_3^2}^{-1}(1-p_A^\prime)\Big)\geq 1 - F_{\chi_3^2}\Big(\frac{1}{(1-\frac{r_1}{\sqrt 2})^4}F_{\chi_3^2}^{-1}(1-p_A^\prime)\Big).
\end{equation}
where $p_A^\prime=F_{\chi_3^2}\Big(\frac{1}{k_3^2} F_{\chi_3^2}^{-1}(p_A)\Big)$.
\begin{equation}
    \inf_{k_1,k_2<1,k_3\geq 1} \tilde p_A =  \inf_{r_1^2 + r_2^2 = R^2, r_1,r_2\geq 0} 1 - F_{\chi_3^2}\left(\frac{1}{(1-\frac{r_1}{\sqrt{2}})^4}F_{\chi_3^2}^{-1}\left(1 - F_{\chi_3^2}\left(\frac{1}{(1+r_2)^2} F_{\chi_3^2}^{-1}(p_A)\right) \right) \right)
\end{equation}
\end{itemize}
As a result, $\tilde p_A \geq \inf_{k_1,k_2,k_3}\tilde p_A = \min\{p_1,p_2,p_3,p_4\}$.
\end{proof}

\cref{lemma:pAtilde} bridge $\tilde p_A$ with $p_A$ with a functional relation. However, the infimum used in the definition of $p_3$ and $p_4$ makes them hard to compute. Fortunately, we can draw samples in $\{(r_1,r_2)|r_1^2 + r_2^2= R^2, r_1,r_2\geq 0\}$ and calculate lower bounds for $p_3$ and $p_4$. Let $r_1 = R\cos\theta $ and $r_2 = R\sin \theta$ where $\theta \in [0,\frac{\pi}{2}]$. Suppose minimized functions are $L_3$ and $L_4$-Lipschitz in terms of $\theta$, respectively. If we sample $\theta_i = (i/\epsilon) \cdot \frac{\pi}{2}$, then $p_3 \geq \min_{i} \mathcal F_3(\theta_i,p_A) - \frac{\epsilon L_3}{2}$ and $p_4 \geq \min_i \mathcal F_4(\theta_i,p_A) - \frac{\epsilon L_4}{2}$. By taking partial derivative on $\mathcal F_3$ and $\mathcal F_4$, the Lipschitz constants are bounded by
\begin{align}
    &L_3 \leq \frac{2Ru}{\sqrt{\pi e}} + \frac{\sqrt{2}R F_{\chi_3^2}^{-1}(1-p_A) }{\sqrt{\pi} (1-R)^3 \sqrt{u} e^{-\frac{u}{2}+1}}.\\
    &L_4 \leq \frac{\sqrt{2}Ru}{\sqrt{\pi e}} + \frac{2R F_{\chi_3^2}^{-1}(1-p_A)}{\sqrt{\pi} (1-\frac{R}{\sqrt{2}})^5 \sqrt{u}e^{-\frac{u}{2}+1}\cdot}, \text{ where } u = F_{\chi_3^2}^{-1}(p_A).
\end{align}

\subsection{Proof of \cref{thm:taper+rotation}: Certifying Z-taper $\circ$ Z-rotation}
\label{appendix:taper+rotation}
\emph{Proof.} Consider the composite transformation $\phi_{TP} \circ \phi_{Rot-z}$ with parameter space $\mathcal Z = \mathcal Z_{TP}\times \mathcal Z_{Rot-z} = \mathbb R^2$. As stated in \cref{thm:taper+rotation}, we sample $\varphi \theta M^2$ parameters $z_{jk} = (\varphi_j,\theta_k)$ in the subspace $S = [-\varphi, \varphi]\times [-\theta, \theta]\subseteq \mathcal Z$, with $\varphi_j = \frac{2j}{M} - \varphi$ and $\theta_k = \frac{2k}{M} - \theta$. The interpolation error of a point cloud $x = \{p_i\}_{i=1}^N$ ($p_i = x_i,y_i,z_i$) between two transformations $z_{jk} = (\varphi_j,\theta_k)$ and $z^\prime = (\varphi^\prime, \theta^\prime)$ is:
\begin{align}
    \mathcal M(z_{jk},z^\prime) &= \left(\sum_{i=1}^N\|\phi_{TP}(\phi_{Rot-z}(p_i,\theta_k),\varphi_j) - \phi_{TP}( \phi_{Rot-z}(p_i,\theta^\prime),\varphi^\prime)\|_2^2 \right)^{1/2}\\
    &= \left(\sum_{i=1}^N\|\phi_{TP}(p_i^\prime,\varphi_j) -\phi_{TP}(\phi_{Rot-z}(p_i^\prime,\theta^\prime-\theta_k),\varphi^\prime)\|_2^2 \right)^{1/2}\ \text{where } p_i^\prime = \phi_{Rot-z}(p_i,\theta_k)\\
    &= \left(\sum_{i=1}^N \Big[(1+\varphi_j z_i^\prime)^2 r_i^{\prime2} + (1+\varphi^\prime z_i^\prime)^2 r_i^{\prime 2} - 2(1+\varphi_j z_i^\prime)(1+\varphi^\prime z_i^\prime)r_i^{\prime 2} \cos(\theta^\prime - \theta_k)\Big] \right)^{1/2} \quad (r_i^{\prime 2} = x_i^{\prime 2} + y_i^{\prime 2})\label{eq:65}\\
    &\leq \left(\sum_{i=1}^N \Big[(\varphi^\prime - \varphi_j)^2 z_i^{\prime2} r_i^{\prime2} + (\theta^\prime - \theta_k)^2r_i^{\prime2}(1+\varphi_j z_i^\prime)(1+\varphi^\prime z_i^\prime)\Big]\right)^{1/2}
\end{align}

\cref{eq:65} uses the law of cosine to compute the $\ell_2$ distance. Note that $\max_{\theta^\prime}\min_{k}|\theta^\prime - \theta_k| = \frac{1}{M}$ and $\max_{\varphi^\prime}\min_j|\varphi^\prime - \varphi_j| = \frac{1}{M}$. Also, $z_i^{\prime2} r_i^{\prime2} \leq \frac{1}{4}$ for $z_i^{\prime2} + r_i^{\prime2} \leq 1$. Therefore, the interpolation error
\begin{align}
    \mathcal M_S &= \max_{z=(\varphi^\prime, \theta^\prime)\in S}\min_{j,k} \mathcal M(z_{jk}, z^\prime)\\
    &\leq \left(\sum_{i=1}^N \big(\frac{1}{4M^2} + \frac{(1+\varphi)^2}{M^2}\big)\right)^{1/2}\\
    &= \frac{\sqrt{N(4\varphi^2 + 8\varphi + 5)}}{2M}
\end{align}
It thus follows from \cref{thm:TSS-DR} that for any $z\in S$, $y_A = \arg\max_y q(y|\phi(x,z);\epsilon)$; if $\ \forall j,k:$
\begin{equation}
    \frac{\sigma}{2}\Big(\Phi^{-1}\Big(p_A^{(jk)}\Big) - \Phi^{-1}\Big(p_B^{(jk)}\Big) \Big) \geq \frac{\sqrt{N(4\varphi^2 + 8\varphi + 5)}}{2M}.
\end{equation}

\subsection{Proof of \cref{thm:twist+taper+rotation}: Certifying Z-twist $\circ$ Z-taper $\circ$ Z-rotation}
\label{appendix:twist+taper+rotation}
\emph{Proof.} The composite transformation $\phi_{Tz} \circ \phi_{TP} \circ \phi_{Rot-z}$ has a parameter space of $\mathcal Z = \mathcal Z_{Twist} \times \mathcal Z_{Taper} \times \mathcal Z_{Rot-z} = \mathbb R^3$. We calculate the interpolation error of a point cloud $x = \{p_i\}_{i=1}^N$ between two transformations $z_{jkl} = (\varphi_j, \alpha_k, \theta_l)$ and $z^\prime = (\varphi^\prime, \alpha^\prime, \theta^\prime)$. (Note that z-twist, z-taper and z-rotation are pairwise commutative.)
\begin{align}
    \mathcal M(z^\prime, z_{jkl}) &= \| \phi(x,z_{jkl}) - \phi(x,z^\prime) \|_2 \\
    &= \left(\sum_{i=1}^N\|\phi(p_i,z_{jkl})-\phi(p_i,z^\prime)\|_2^2 \right)^{1/2}\\
    &= \left(\sum_{i=1}^N\|  
    \phi_{TP}(p_i^\prime, \alpha_k) - \phi_{TP}(\phi_{Tz}(\phi_{Rot-z}(p_i^\prime,\theta^\prime - \theta_l),\varphi^\prime - \varphi_j),\alpha^\prime)
    \|_2^2 \right)^{1/2}\\
    &= \left(\sum_{i=1}^N \left[(1+\alpha_k z_i^\prime)^2 r_i^{\prime2} + (1+\alpha^\prime z_i^\prime)^2 r_i^{\prime2} - 2(1+\alpha_k z_i^\prime)(1+\alpha^\prime z_i^\prime)r_i^{\prime 2} \cos ((\varphi^\prime - \varphi_j)z_i^\prime + \theta^\prime - \theta_l) \right] \right)^{1/2}\label{eq:74}\\
    &\leq \left(\sum_{i=1}^N \left[(\alpha_k - \alpha^\prime)^2 z_i^{\prime2}r_i^{\prime2} + (1+\alpha_k z_i^\prime)(1+\alpha^\prime z_i^\prime)((\varphi^\prime - \varphi_j)z_i^\prime + \theta^\prime - \theta_l)^2 r_i^{\prime2}\right] \right)^{1/2}
\end{align}

\cref{eq:74} uses the law of cosine for computing the $\ell_2$ distance. Following the sampling strategy, for $z^\prime = (\varphi^\prime, \alpha^\prime, \theta^\prime)\in S = [-\varphi,\varphi, - \alpha, \alpha, -\theta, \theta]$, we have $\max_{\varphi^\prime}\min_j |\varphi^\prime - \varphi_j| = \frac{1}{M}$, $\max_{\alpha^\prime}\min_k |\alpha^\prime - \alpha_k| = \frac{1}{M}$ and $\max_{\theta^\prime}\min_{l}|\theta^\prime - \theta_l| = \frac{1}{M}$. Hence, the maximum interpolation error for the subspace $S$ is
\begin{align}
    \mathcal M_S &= \max_{z^\prime \in S}\min_{j,k,l} \mathcal M(z^\prime, z_{jkl})\\
    &\leq \left(\sum_{i=1}^N \left[ \frac{z_i^{\prime2}r_i^{\prime2}}{M^2} + \frac{(1+\alpha z_i^\prime)^2 (1+z_i^\prime)^2 r_i^{\prime2}}{M^2} \right]\right)^{1/2}\\
    &\leq \left(\sum_{i=1}^N \Big(\frac{1}{4M^2} + \frac{(1+\alpha)^2\times \frac{27}{16}}{M^2}\Big)\right) = \frac{\sqrt{N(1+\frac{27}{4}(1+\alpha^2))}}{2M}
\end{align}

Applying \cref{thm:TSS-DR}, it is guaranteed that for any $z\in S$, $y_A = \arg\max_y(q|\phi(x,z);\epsilon), $ if $\ \forall j,k,l$,
\begin{equation}
    \frac{\sigma}{2}\Big(\Phi^{-1}\Big(p_A^{(jkl)}\Big) - \Phi^{-1}\Big(p_B^{(jkl)}\Big) \Big) \geq \frac{\sqrt{N(1 + \frac{27}{4}(1+\alpha)^2)}}{2M}
\end{equation}

\section{Discussion on $\ell_p$ Norm Bounded Perturbations}
\label{appendix:ellp}
We exhibit the certified robust accuracy under attacks with restricted $\ell_p$ norm in \cref{table:ell_p}. Our \method framework is directly applicable for certifying $\ell_2$ norm bounded attacks. We smooth a base classifier by additive noise $\epsilon\sim \mathcal N(0,\mathds 1_{3\times N})$ so its class probability is $q(y|x;\epsilon) = \mathbb E_{\epsilon} p(y|x+\epsilon)$. Since additive noise is a additive transformation, the smoothed classifier must be robust for any attacks $\alpha\in \mathbb R^{3\times N}$ with
\begin{equation}
    \|\alpha\|_2 \leq \frac{\sigma}{2}\Big(\Phi^{-1}\Big(p_A\Big) - \Phi^{-1}\Big(p_B\Big)\Big).
\end{equation}

Though our \method framework cannot directly be applied to certify against $\ell_\infty$ norm bounded perturbations, we can still derive a certification bound for point clouds $x\in \mathbb R^{3\times N}$ by a loose relaxation $\|\theta\|_\infty \leq \sqrt{3N}\|\theta\|_2$. In fact, this certification bound for $\ell_\infty$ is the best we can get in terms of dimension dependence, as \cite{yang2020randomized} pointed out that no smoothing techniques can certify nontrivial accuracy within a radius of $\Omega(N^{-1/2})$. However, this relaxation is too imprecise when applying it in a high-dimensional space. As a result, the certified accuracy for $\ell_\infty$ norm drops as the point cloud size increases.

\begin{table}[t]
\centering
\caption{Certified robustness for point cloud models under different $\ell_p$ attacks. We achieve similar certification bound for $\ell_\infty$ norm bounded attack as DeepG3D \cite{DeepG3D}.}
\vskip 0.15in
\label{table:ell_p}
\begin{small}
    \begin{tabular}{cccccccc}
    \toprule
    \multirow{2}{*}{Attack} & \multirow{2}{*}{Radius} & \multicolumn{3}{c}{\method} & &\multicolumn{2}{c}{DeepG3D}  \\
    \cline{3-5} \cline{7-8}\\[-7pt]
    & & 16 & 64 & 256 & & 64 & 256\\
    \midrule
    $\ell_2$ & 0.05 & 74.1 & 82.2 & \textbf{84.2} & &-&-\\
    $\ell_2$ & 0.1 & 61.9 & 70.8 & \textbf{77.3} & &-&- \\
    $\ell_\infty$ & 0.01 & \textbf{70.9} & 64.4 & 47.0 & & \textbf{70.9} & 67.0\\
    \bottomrule
    \end{tabular}
    \end{small}
\vskip -0.1in
\end{table}

\section{Implementation Details}

In the implementation of \method, we need to figure out $p_A$, $p_B$~(for additive transformations), or $p_A^{(i)}$, $p_B^{(i)}$ for other transformations.
Following convention~\citep{cohen19c}, we set $p_B = 1 - p_A$ which gives an upper bound of true $p_B$ and thus resulting in a sound relaxation.
To figure out $p_A$, we use Monte-Carlo sampling and Clopper-Pearson confidence interval~\citep{clopper1934use} to obtain a high-confidence lower bound~(denoted as $\underline{p_A}$) of $p_A$, which implies a high-confidence robustness certification.
Specifically, we sample $N=10^3$ times for additive transformations and set confidence level $1-\alpha = 99.9\%$.
Regarding other transformations, to guarantee that overall certification holds with confidence level $99.9\%$, suppose there are $M$ samples of $z_j$, then for each sampled parameter $z_j$, the $p_A^{(j)}$ estimation uses $N=10^4$ samples with confidence level $(1-\alpha/M)$.


\section{More Experimental Details}

\subsection{Runtime}

One drawback of randomized smoothing techniques is their extra computation overheads for certification. In our implementation, all transformations are processed in batch to accelerate computation. We report the average inference time of \method for each point cloud input in \cref{table:runtime}

\begin{table}[t]
\vspace{-0.7em}
\centering
\caption{\small Inference time of TPC for different transformations}
\label{table:runtime}
\begin{tabular}{llc}
\toprule
    Transformation & Attack radius & Inference time ($\mathrm {ms}$)\\
    \midrule
    General rotation & 15$^\circ$ & 9.56\\
    Z-rotation & 180$^\circ$ & 3.12\\
    Z-shear & 0.2 & 3.46\\
    Z-twist & 180$^\circ$ & 3.76\\
    Z-taper & 0.5 & 10.61 \\
    Linear & 0.2 & 3.94 \\
    Z-twist $\circ$ Z-rotation & 50$^\circ$, 5$^\circ$& 5.34\\
    Z-taper $\circ$ Z-rotation & 0.2, 1$^\circ$ & 8.72\\
    Z-twist $\circ$ Z-taper $\circ$ Z-rotation & 20$^\circ$, 0.2, 1$^\circ$ & 9.40\\
    \bottomrule
\end{tabular}
\end{table}

\subsection{Certified Ratio}
\label{appendix:certified-ratio}

The certified ratio is defined as the fraction of test point clouds classified \emph{consistently}, but not necessarily \emph{correctly} under a set of attacks. We compare the certified ratio achieved by our \method method with the baseline, DeepG3D \cite{DeepG3D} in \cref{table:ratio}.
\begin{table}[!t]
    \centering
    \caption{Comparison of certified ratio achieved by our transformation-specific smoothing framework \method and the baseline, DeepG3D \cite{DeepG3D}. ``-" denotes the settings where the baselines cannot scale up to.}
    \label{table:ratio}
    \vskip 0.15in
    \begin{small}
    \begin{tabular}{llcc}
    \toprule
    \multirow{2}{*}[-2pt]{Transformation} & \multirow{2}{*}[-2pt]{Attack radius} & \multicolumn{2}{c}{Certified Ratio ($\%$)} \\[1pt] \cline{3-4} \\[-7pt]
    &           & TPC &  DeepG3D   \\
    \midrule
    \multirow{2}{*}{ZYX-rotation} & $2^\circ$ & \textbf{92.6} & 72.8\\
    & $5^\circ$ & \textbf{79.5} & 58.7 \\
    \midrule
    \multirow{3}{*}{General rotation} & $5^\circ$ & \textbf{89.4} & - \\
    & $10^\circ$ & \textbf{79.5} & -  \\
    & $15^\circ$ & \textbf{63.1} & -  \\ 
    \midrule
    \multirow{3}{*}{Z-rotation} & $20^\circ$ & \textbf{99.0}  & 96.7\\
    & $60^\circ$ & \textbf{98.1}  & 95.7 \\
    & $180^\circ$& \textbf{95.2} & - \\
    \midrule
    \multirow{3}{*}{Z-shear} & 0.03 & \textbf{98.6} &  70.7 \\
    &  0.1  & \textbf{97.1} & - \\
    & 0.2 & \textbf{91.8} & - \\ 
    \midrule
    \multirow{3}{*}{Z-twist} & $20^\circ$ & \textbf{100.0} & 23.9 \\
    & $60^\circ$ & \textbf{95.6} & - \\
    & $180^\circ$ & \textbf{77.5} & - \\
    \midrule
    \multirow{3}{*}{Z-taper} & 0.1 & \textbf{95.2} & 81.5 \\
    & 0.2 & \textbf{93.3} & 28.3 \\
    & 0.5 & \textbf{91.2} & - \\ 
    \midrule
    \multirow{3}{*}{\makecell[l]{Z-twist $\circ$\\Z-rotation}} & $20^\circ$, $1^\circ$& \textbf{96.5} & 16.3 \\
    & $20^\circ$, $5^\circ$& \textbf{96.0} & - \\
    & $50^\circ$, $5^\circ$& \textbf{95.0} & -  \\ 
    \midrule
    \multirow{2}{*}{\makecell[l]{Z-taper $\circ$\\Z-rotation}}
    & 0.1, $1^\circ$ & \textbf{89.5} & 68.5  \\
    & 0.2, $1^\circ$ & \textbf{86.1} & 20.7\\ 
    \midrule
    \multirow{2}{*}{\makecell[l]{Z-twist $\circ$ Z-taper\\$\circ$ Z-rotation}} & $10^\circ$, 0.1, $1^\circ$& \textbf{74.9} & 20.7 \\
    &$20^\circ$, 0.2, $1^\circ$ & \textbf{68.7} & 5.4\\
    \bottomrule
    \end{tabular}
    \end{small}
    \vskip -0.1in
\end{table}

\subsection{Evaluation of other architectures}
\label{appendix:curvenet}
In \cref{sec:exp} of the main text, we only discuss one particular architecture for the point cloud model for clarity of comparing with the baseline \cite{DeepG3D}. However, \method is model-agnostic and can be directly applied to other architectures. We also conduct experiments on a more complicated architecture, CurveNet \cite{curvenet} to show the flexibility and scalability of \method. The certified robust accuracy for CurveNet is shown in \cref{table:curvenet}.

\begin{table}[!h]
    \vspace{-0.5em}
\caption{\small Certified robust accuracy of TPC on CurveNet}
\label{table:curvenet}
    \centering
    \begin{tabular}{llcc}
    \toprule
    \multirow{2}{*}[-2pt]{Transformation} & \multirow{2}{*}[-2pt]{Attack radius} & \multicolumn{2}{c}{Certified Accuracy ($\%$)} \\[1pt] \cline{3-4}\\[-10pt]
    &           & PointNet&  CurveNet \\
    \midrule
    Z-rotation  & 180$^\circ$ & 81.3& \textbf{85.4}\\
    Z-shear  & 0.2  & 77.7 & \textbf{87.8}\\
    Z-twist & 180$^\circ$ & 64.3 & \textbf{86.2}\\
    Z-taper & 0.2 &  76.5 & \textbf{88.6}\\
    Linear & 0.2 & 59.9 & \textbf{77.7} \\
    \bottomrule
    \end{tabular}
    \label{table2}
\end{table}


\end{document}